\documentclass{article}
\usepackage[utf8]{inputenc} 
\usepackage[T1]{fontenc}    
\usepackage{hyperref}       
\usepackage{xcolor}
\hypersetup{
    colorlinks,
    linkcolor={red!50!black},
    citecolor={blue!50!black},
    urlcolor={blue!80!black}
}
\usepackage{url}            
\usepackage{booktabs}       
\usepackage{amsfonts}       
\usepackage{nicefrac}       
\usepackage{microtype}      
\usepackage{todonotes}
\usepackage{tikz}
\usepackage{times}

\usepackage{amsmath}

\def\eqref#1{equation~\ref{#1}}

\def\1{\bm{1}}

\DeclareMathAlphabet{\mathsfit}{\encodingdefault}{\sfdefault}{m}{sl}
\SetMathAlphabet{\mathsfit}{bold}{\encodingdefault}{\sfdefault}{bx}{n}

\def\gA{{\mathcal{A}}}

\def\gE{{\mathcal{E}}}

\def\gL{{\mathcal{L}}}

\def\gN{{\mathcal{N}}}

\def\gQ{{\mathcal{Q}}}

\def\gX{{\mathcal{X}}}

\def\sP{{\mathbb{P}}}

\newcommand{\E}{\mathbb{E}}

\newcommand{\R}{\mathbb{R}}

\newcommand{\KL}{D_{\mathrm{KL}}}

\usepackage{amsmath}
\DeclareMathOperator*{\argmax}{arg\,max}

\newcommand{\e}{\epsilon}
\newcommand{\norm}[1]{\left\lVert#1\right\rVert}

\usepackage{bm}
\usepackage[normalem]{ulem}
\usepackage{todonotes}
\usepackage{enumitem}
\usepackage{amsthm}

\newtheorem{lemma}{Lemma} 
\newtheorem{proposition}{Proposition}

\newtheorem{assumption}{Assumption}

\usepackage{subfigure}
\usepackage{graphicx}
\usepackage{cancel}

\usepackage{siunitx}

\usepackage{thm-restate}

\usepackage[accepted]{icml2020}
\icmltitlerunning{Augmented Normalizing Flows}

\begin{document}

\twocolumn[
\icmltitle{Augmented Normalizing Flows: \linebreak
Bridging the Gap Between Generative Flows and Latent Variable Models}

\begin{icmlauthorlist}
\icmlauthor{Chin-Wei Huang}{mil}
\icmlauthor{Laurent Dinh}{goo}
\icmlauthor{Aaron Courville}{mil,cif}
\end{icmlauthorlist}

\icmlaffiliation{mil}{Mila}
\icmlaffiliation{goo}{Google}
\icmlaffiliation{cif}{CIFAR fellow}

\icmlcorrespondingauthor{Chin-Wei Huang}{chin-wei.huang@umontreal.ca}

\icmlkeywords{Normalizing flows, generative flows, variational autoencoders, latent variable models, Hamiltonian ODE, universal approximation, invertible neural networks}

\vskip 0.3in
]

\printAffiliationsAndNotice{}  

\begin{abstract}
In this work, we propose a new family of generative flows
on an augmented data space,
with an aim to improve expressivity 
without drastically increasing the computational cost of sampling and evaluation of a lower bound on the likelihood.
Theoretically, we prove the proposed flow can approximate a Hamiltonian ODE as a universal transport map. 
Empirically, we demonstrate state-of-the-art performance on standard benchmarks of flow-based generative modeling. 
\end{abstract}

\section{Introduction}
\begin{figure*}
    \centering
    \includegraphics[width=1.0\textwidth]{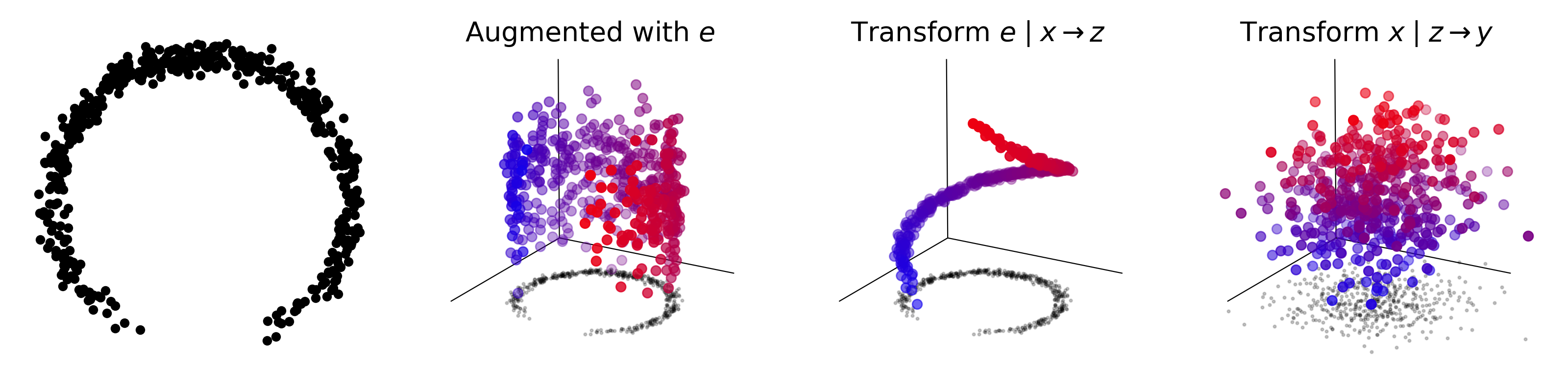}
    \vspace{-0.7cm}
    \caption{\small 
    Transforming data $x$ (\emph{left}) via augmented normalizing flows:
    Black dots and blue dots are marginal and joint data points, respectively.
    \emph{First step}: augment the data $x$ with an independent noise $e$.
    \emph{Second step}: transform the augmented data $e$ conditioned on $x$ into $z$. 
    \emph{Third step}: transform the original data $x$ conditioned on $z$ into $y$, resulting in a Gaussianized joint distribution of $(y,z)$}
    \label{fig:2d_vis}
\end{figure*}
Deep invertible models have recently gained increasing interest among machine learning researchers as they constitute a powerful probabilistic toolkit. 
They allow for the tracking of  changes in probability density and have been widely applied in many tasks, including
\begin{enumerate}[label=(\roman*)]
    \item generative models~\citep{dinh2016density, kingma2018glow, chen2019residualflows},
    \item variational inference~\citep{rezende2015variational,kingma2016improved,berg2018sylvester},
    \item density estimation~\citep{papamakarios2017masked,huang2018neural},
    \item reinforcement learning~\citep{mazoure2019leveraging,ward2019improving}, etc. 
\end{enumerate} 

The main challenges in designing an invertible model for the above use cases are
to ensure (1) the mapping $f$ is invertible, (2) the log-determinant of the Jacobian of $f$ is cheap to compute, and (3) $f$ is expressive. 
For use case (i), ideally we would also like to (4) invert $f$ efficiently. 

In general, it is hard to design a family of functions that satisfy all of the above. 
Most work within this line of research is dedicated to improving the expressivity of the bijective mapping while maintaining the computational tractability of the log-determinant of Jacobian \citep{dinh2016density, kingma2016improved, huang2018neural, chen2019residualflows}. 

Aside from the unfortunate trade-off between the computational budget of inversion/Jacobian log-determinant and the expressivity of the invertible mapping, generative flows suffer from the limitation of local dependency. 
Unlike latent variable models such as Variational Autoencoders (VAEs; \citealt{kingma2013auto, rezende2014stochastic}) and Generative Adversarial Networks (GANs; \citealt{goodfellow2014generative}) which model the high dimensional data as coordinates in another space, most generative flows model the dependency among features only locally. 
Dependencies of features far away from each other can only be propagated through composition of mappings, which progressively enlarges the receptive field. 
Special design of parameterization like the attention mechanism can be made to address this issue \citep{ho2019flow++}.

In this paper, we propose to construct an invertible model on an augmented input space, which when combined with the block-wise coupling of \citet{dinh2016density} satisfies all criteria (1-4). 
The motivation is that to transform some distribution (such as the marginal distribution of $x$ pictured in Figure~\ref{fig:2d_vis}) into another (e.g. standard normal) in the original input space, $f$ needs to be capable of transporting the probability mass ``non-uniformly'' across its domain, 
whereas in an augmented input space it is possible to find a smoother transformation.
For instance, if we couple the data $x$ with an independent noise $e$, we can first transform $e$ conditioned on $x$ into $z$, so that conditioned on different values of $z$, $x$ can be more easily centered and Gaussianized. 
Our proposed method also generalizes multiple variants of VAEs and possesses the advantage of transforming the data in a more globally coherent manner via first embedding the data in the augmented state space.
Finally, operating on an augmented state space allows us to sidestep the topology preserving property of a diffeomorphism, which means input space can potentially be more freely deformed \citep{dupont2019augmented}. 

\paragraph{Our contributions:}
we introduce \emph{Augmented Normalizing Flows} (ANFs), an invertible generative model on the real-valued data $x$ coupled with an independent noise $e$. 
We propose a parameter estimation principle called \emph{Augmented Maximum Likelihood Estimation} (AMLE), which we show amounts to maximizing a lower bound on the marginal likelihood of the original data $x$. 
Theoretically, we show that the family of ANFs with additive coupling can universally transform arbitrary data distribution into a standard Gaussian prior, augmented with a degenerate deterministic variable. 
To the best of our knowledge, this is the first attempt in understanding how expressivity can be improved via composing flow layers  rather than widening the flow \citep{huang2018neural}. 
Experimentally, we apply the proposed method to a suite of standard generative modelling tasks and demonstrate state-of-the-art performance in density estimation and image synthesis.

\section{Background}
\label{sec:background}
Given a training set $(x_i)_{i=1}^n \sim q(x)^n$, where $x_i\in\gX$, and a family of density models $\{p_\pi(x): \pi\in\mathfrak{P}(\gX)\}$, where $\mathfrak{P}(\gX)$ is a collection of sets of parameters that can sufficiently describe the density function, the \emph{Maximum Likelihood Principle} estimates the parameters by maximizing the chance of the data being generated by the assumed model:
\begin{align}
\hat{\pi}:=\argmax_{\pi\in\mathfrak{P}(\gX)}\left\{ \sum_{i=1}^n\log p_\pi(x_i)\right\} 
=\argmax_{\pi\in\mathfrak{P}(\gX)}\, \E_{x}[\log p_\pi(x)] 
\label{eq:mle}
\end{align}
where the latter expectation is over the empirical distribution $\hat{q}(x)$ ($x_i$ with uniformly distributed random index $i\in\{1,\cdots,n\}$). 
$\hat{\pi}$ is known as the maximum likelihood estimate (MLE) for the parameter $\pi$. 
Below, we review two families of likelihood-based density models.

\paragraph{1. Invertible Generative Models}
Assume $y\sim \gN(0,I)$.
Assume the data is generated via a bijective mapping $x=f_\theta(y)$. 
Then the probability density function of $f_\theta(y)$ evaluated at $x$ can be written as
\begin{align}
p_\theta(x)=\gN(f_\theta^{-1}(x); 0,I)\left|\det\frac{\partial f_\theta^{-1}(x)}{\partial x}\right|
\label{eq:cov}
\end{align}
Equivalently, one can parameterize the inverse transformation $x\mapsto g_\theta(x)$ with invertible mapping $g_\theta$, and define the generative transformation as $f_\theta=g_\theta^{-1}$.

Much of the design effort has been dedicated to ensuring (1) the invertibility of the transformation $g$, and (2) efficiency in computing the log-determinant of the Jacobian in Equation~\ref{eq:cov}. 
For example, \citet{dinh2016density} propose the affine coupling:
$$g_\theta(x_a, x_b) = \texttt{concat}(x_a,\, s_\theta(x_a)\odot x_b + m_\theta(x_a))$$
where $s_\theta$ and $m_\theta$ are parameterized by neural networks and $x_a$ and $x_b$ are two partitioning of the data vector, 
and compose multiple layers of transformations intertwined with permutation of elements of $x$.

Invertible models allow for exact computation of the likelihood, and can be composed to increase modelling capacity. 
Nevertheless, the expressivity of the transformation is limited due to the need to satisfy invertibility and to reduce the cost of computing the Jacobian determinant. 
For a comprehensive review of this topic, see \citet{kobyzev2019normalizing} and \citet{papamakarios2019normalizing}.

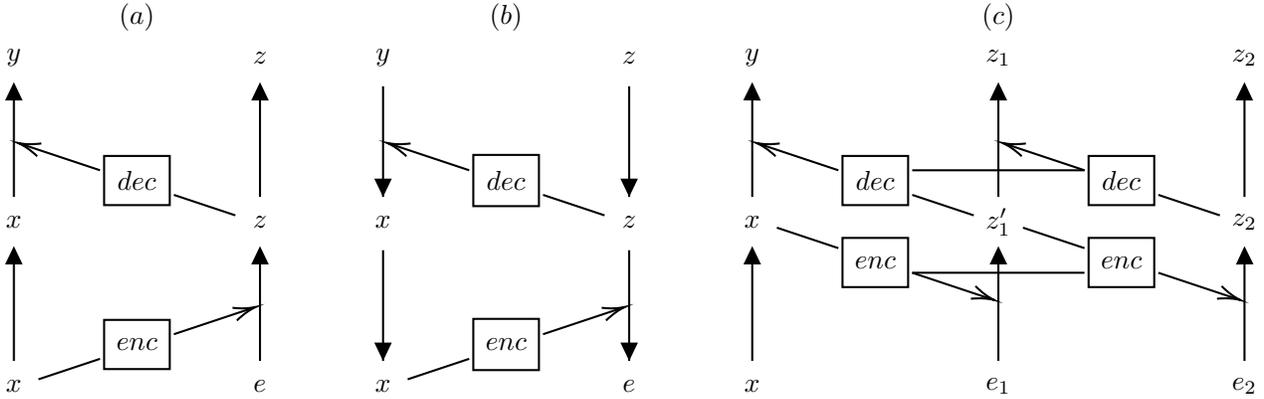
\begin{figure*}
    \centering
    \resizebox{1.0\textwidth}{!}{%
    \tikzset{every picture/.style={line width=0.75pt}} 

\begin{tikzpicture}[x=0.75pt,y=0.75pt,yscale=-1,xscale=1]


\draw (70,0) node   {$(a)$};

\draw (10,180) node   {$x$};
\draw (10,100) node   {$x$};
\draw (10,20) node   {$y$};
\draw (130,180) node   {$e$};
\draw (130,100) node   {$z$};
\draw (130,20) node   {$z$};

\draw (70,160) node   {$enc$};
\draw (70,80) node   {$dec$};
\draw    (54,148) -- (86,148) -- (86,172) -- (54,172) -- cycle  ;
\draw    (54,68) -- (86,68) -- (86,92) -- (54,92) -- cycle  ;

\draw    (10,168) -- (10,114) ;
\draw [shift={(10,112)}, rotate = 90] [fill={rgb, 255:red, 0; green, 0; blue, 0 }  ][line width=0.75]  [draw opacity=0] (8.93,-4.29) -- (0,0) -- (8.93,4.29) -- cycle    ;
\draw    (130,168) -- (130,114) ;
\draw [shift={(130,112)}, rotate = 90] [fill={rgb, 255:red, 0; green, 0; blue, 0 }  ][line width=0.75]  [draw opacity=0] (8.93,-4.29) -- (0,0) -- (8.93,4.29) -- cycle    ;
\draw    (10,88) -- (10,34) ;
\draw [shift={(10,32)}, rotate = 450] [fill={rgb, 255:red, 0; green, 0; blue, 0 }  ][line width=0.75]  [draw opacity=0] (8.93,-4.29) -- (0,0) -- (8.93,4.29) -- cycle    ;
\draw    (130,88) -- (130,34) ;
\draw [shift={(130,32)}, rotate = 450] [fill={rgb, 255:red, 0; green, 0; blue, 0 }  ][line width=0.75]  [draw opacity=0] (8.93,-4.29) -- (0,0) -- (8.93,4.29) -- cycle    ;


\draw    (22,177) -- (52,167) ;
\draw    (88,155) -- (127,142) ;
\draw [shift={(128,141.67)}, rotate = 521.1600000000001] [color={rgb, 255:red, 0; green, 0; blue, 0 }  ][line width=0.75]    (10.93,-3.29) .. controls (6.95,-1.4) and (3.31,-0.3) .. (0,0) .. controls (3.31,0.3) and (6.95,1.4) .. (10.93,3.29)   ;

\draw    (118,97) -- (88,87) ;
\draw    (52,75) -- (13,62) ;
\draw [shift={(12,61.67)}, rotate = 377.91999999999996] [color={rgb, 255:red, 0; green, 0; blue, 0 }  ][line width=0.75]    (10.93,-3.29) .. controls (6.95,-1.4) and (3.31,-0.3) .. (0,0) .. controls (3.31,0.3) and (6.95,1.4) .. (10.93,3.29)   ;

\draw (250,0) node   {$(b)$};

\draw (190,180) node   {$x$};
\draw (190,100) node   {$x$};
\draw (190,20) node   {$y$};
\draw (310,180) node   {$e$};
\draw (310,100) node  {$z$};
\draw (310,20) node   {$z$};

\draw (250,160) node   {$enc$};
\draw (250,80) node   {$dec$};

\draw    (234,147.5) -- (266,147.5) -- (266,172.5) -- (234,172.5) -- cycle  ;
\draw    (234,67.5) -- (266,67.5) -- (266,92.5) -- (234,92.5) -- cycle  ;

\draw    (190,114) -- (190,168) ;
\draw [shift={(190,168)}, rotate = 270] [fill={rgb, 255:red, 0; green, 0; blue, 0 }  ][line width=0.75]  [draw opacity=0] (8.93,-4.29) -- (0,0) -- (8.93,4.29) -- cycle    ;
\draw    (310,114) -- (310,168) ;
\draw [shift={(310,168)}, rotate = 270] [fill={rgb, 255:red, 0; green, 0; blue, 0 }  ][line width=0.75]  [draw opacity=0] (8.93,-4.29) -- (0,0) -- (8.93,4.29) -- cycle    ;
\draw    (310,88) -- (310,34) ;
\draw [shift={(310,88)}, rotate = 270] [fill={rgb, 255:red, 0; green, 0; blue, 0 }  ][line width=0.75]  [draw opacity=0] (8.93,-4.29) -- (0,0) -- (8.93,4.29) -- cycle    ;
\draw    (190,88) -- (190,34) ;
\draw [shift={(190,88)}, rotate = 270] [fill={rgb, 255:red, 0; green, 0; blue, 0 }  ][line width=0.75]  [draw opacity=0] (8.93,-4.29) -- (0,0) -- (8.93,4.29) -- cycle    ;


\draw    (202,177) -- (232,167) ;
\draw    (268,155) -- (307,142) ;
\draw [shift={(308,141.67)}, rotate = 521.1600000000001] [color={rgb, 255:red, 0; green, 0; blue, 0 }  ][line width=0.75]    (10.93,-3.29) .. controls (6.95,-1.4) and (3.31,-0.3) .. (0,0) .. controls (3.31,0.3) and (6.95,1.4) .. (10.93,3.29)   ;

\draw    (298,97) -- (268,87) ;
\draw    (232,75) -- (193,62) ;
\draw [shift={(192,61.67)}, rotate = 377.91999999999996] [color={rgb, 255:red, 0; green, 0; blue, 0 }  ][line width=0.75]    (10.93,-3.29) .. controls (6.95,-1.4) and (3.31,-0.3) .. (0,0) .. controls (3.31,0.3) and (6.95,1.4) .. (10.93,3.29)   ;


\draw (490,0) node   {$(c)$};

\draw (370,180) node   {$x$};
\draw (370,100) node   {$x$};
\draw (370,20) node   {$y$};

\draw (490,180) node   {$e_{1}$};
\draw (490,100) node  {$z_{1}'$};
\draw (490,20) node   {$z_{1}$};

\draw (610,180) node   {$e_{2}$};
\draw (610,100) node   {$z_{2}$};
\draw (610,20) node   {$z_{2}$};

\draw (430,120) node   {$enc$};
\draw (550,120) node   {$enc$};
\draw (430,80) node   {$dec$};
\draw (550,80) node   {$dec$};
\draw    (414,108) -- (446,108) -- (446,132) -- (414,132) -- cycle  ;
\draw    (534,108) -- (566,108) -- (566,132) -- (534,132) -- cycle  ;
\draw    (414,68) -- (446,68) -- (446,92) -- (414,92) -- cycle  ;
\draw    (534,68) -- (566,68) -- (566,92) -- (534,92) -- cycle  ;

\draw    (370,168) -- (370,114) ;
\draw [shift={(370,112)}, rotate = 90] [fill={rgb, 255:red, 0; green, 0; blue, 0 }  ][line width=0.75]  [draw opacity=0] (8.93,-4.29) -- (0,0) -- (8.93,4.29) -- cycle    ;
\draw    (490,168) -- (490,114) ;
\draw [shift={(490,112)}, rotate = 89.75] [fill={rgb, 255:red, 0; green, 0; blue, 0 }  ][line width=0.75]  [draw opacity=0] (8.93,-4.29) -- (0,0) -- (8.93,4.29) -- cycle    ;
\draw    (610,168) -- (610,114) ;
\draw [shift={(610,112)}, rotate = 89.24] [fill={rgb, 255:red, 0; green, 0; blue, 0 }  ][line width=0.75]  [draw opacity=0] (8.93,-4.29) -- (0,0) -- (8.93,4.29) -- cycle    ;

\draw    (370,88) -- (370,34) ;
\draw [shift={(370,32)}, rotate = 450.66] [fill={rgb, 255:red, 0; green, 0; blue, 0 }  ][line width=0.75]  [draw opacity=0] (8.93,-4.29) -- (0,0) -- (8.93,4.29) -- cycle    ;
\draw    (490,88) -- (490,34) ;
\draw [shift={(490,32)}, rotate = 449.95] [fill={rgb, 255:red, 0; green, 0; blue, 0 }  ][line width=0.75]  [draw opacity=0] (8.93,-4.29) -- (0,0) -- (8.93,4.29) -- cycle    ;
\draw    (610,88) -- (610,34) ;
\draw [shift={(610,32)}, rotate = 450.49] [fill={rgb, 255:red, 0; green, 0; blue, 0 }  ][line width=0.75]  [draw opacity=0] (8.93,-4.29) -- (0,0) -- (8.93,4.29) -- cycle    ;


\draw    (382,103) -- (412,113) ;
\draw    (448,125) -- (487,138) ;
\draw [shift={(488,138.33)}, rotate = 200] [color={rgb, 255:red, 0; green, 0; blue, 0 }  ][line width=0.75]    (10.93,-3.29) .. controls (6.95,-1.4) and (3.31,-0.3) .. (0,0) .. controls (3.31,0.3) and (6.95,1.4) .. (10.93,3.29)   ;

\draw    (502,103) -- (532,113) ;
\draw    (568,125) -- (607,138) ;
\draw [shift={(608,138.33)}, rotate = 200] [color={rgb, 255:red, 0; green, 0; blue, 0 }  ][line width=0.75]    (10.93,-3.29) .. controls (6.95,-1.4) and (3.31,-0.3) .. (0,0) .. controls (3.31,0.3) and (6.95,1.4) .. (10.93,3.29)   ;

\draw    (598,97) -- (568,87) ;
\draw    (532,75) -- (493,62) ;
\draw [shift={(492,61.67)}, rotate = 377.91999999999996] [color={rgb, 255:red, 0; green, 0; blue, 0 }  ][line width=0.75]    (10.93,-3.29) .. controls (6.95,-1.4) and (3.31,-0.3) .. (0,0) .. controls (3.31,0.3) and (6.95,1.4) .. (10.93,3.29)   ;

\draw    (478,97) -- (448,87) ;
\draw    (412,75) -- (373,62) ;
\draw [shift={(372,61.67)}, rotate = 377.91999999999996] [color={rgb, 255:red, 0; green, 0; blue, 0 }  ][line width=0.75]    (10.93,-3.29) .. controls (6.95,-1.4) and (3.31,-0.3) .. (0,0) .. controls (3.31,0.3) and (6.95,1.4) .. (10.93,3.29)   ;

\draw    (448,125) -- (532,125) ;
\draw    (448,75) -- (532,75) ;

\end{tikzpicture}
    }
    \vspace{-0.5cm}
    \caption{\small 
    (\emph{a}) Augmented normalizing flow with block coupling and (\emph{b}) the reverse path for generation. 
    (\emph{c}) Hierarchical augmented normalizing flow. 
    The horizontal connections indicate deterministic features that will be concatenated with the stochastic features in the next transform block.}
    \label{fig:anfs}
\end{figure*}

\paragraph{2. Variational Autoencoders}
Assume the data follows the generating process: $x\sim p_\theta(x|z)$ where $z\sim p_\theta(z)$. 
For simplicity, we assume $p_\theta(z)$ is the standard Gaussian distribution and drop the dependency on $\theta$ henceforward. 
Our goal is to find the MLE for $\theta$, but the log marginal density $\log p_\theta(x) = \log \int_z p_\theta(x|z)p(z) dz$ is generally not tractable since it involves integration. 
Instead, one can maximize a surrogate objective known as the \emph{evidence lower bound} (ELBO):
\begin{align}
\gL(\theta,\phi; x) = \E_{q_\phi(z|x)}[\log p_\theta(x|z)] - \KL(q_\phi(z|x)||p(z))
\end{align}
where $q_\phi(z|x)$ is an inference network that amortizes the cost of parameterizing the variational distribution per input instance $x$ via conditioning. 
Learning and inference can be jointly achieved by drawing a stochastic estimate of the gradient of the ELBO via reparameterization (i.e. change of variable):
\begin{align}\gL(\theta,\phi; x)&=\E_{e\sim q(e)}[\log p_\theta(x|g_\phi(x,e)) \, + \label{eq:elbo_rep}\\
&\log \gN(g_\phi(x,e); 0,I) - \log q_\phi(g_\phi(x,e)|x)] \nonumber\end{align}
if $g_\phi(x,e)$ with $e\sim q(e)$ follows the same density as $q_\phi(z|x)$. Conventionally, $q_\phi(z|x)$ is a multivariate Gaussian distribution with diagonal covariance. 
We write it as $\gN(z; \mu_\phi(x), \sigma^2_\phi(x))$. 
One choice of reparameterization is $g_\phi(x,e)=\mu_\phi(x)+\sigma_\phi(x)\odot e$ with $q(e)=\gN(0,I)$.

VAEs allow one to embed the data in another space (usually of lower dimensionality), and to generate via an arbitrarily parameterized mapping. 
However, the log likelihood of the data is no longer tractable, so we can only maximize an approximate log likelihood.
The performance of the model highly depends on the choice of the encoding distribution and the decoding distribution, as they are closely related to the tightness of the lower bound \citep{cremer2018inference}.

\section{Augmented Maximum Likelihood}
\label{sec:amle}
For augmented maximum likelihood, we couple each data point with an independent random variable $e\in\gE$ drawn from $q(e)$ (in all our experiments we set $q(e)=\gN(0,I)$), and consider a family of joint density models $\{p_\pi(x,e): \pi\in\mathfrak{P}(\gX\times\gE)\}$. 
Instead of maximizing the marginal likelihood of $x_i$'s, we maximize the joint likelihood:
\begin{align}
\hat{\pi}_\gA:=
\argmax_{\pi\in\mathfrak{P}(\gX\times\gE)}\, \E_{x,e}[\log p_\pi(x,e)]
\label{eq:amle}
\end{align}
where the expectation is over $(x,e)\sim \hat{q}(x)q(e)$.
We refer to this extremum estimator as the \emph{Augmented Maximum Likelihood Estimator} (AMLE).
The benefit of maximizing the joint likelihood is that it allows us to make use of the augmented state space to induce structure on the marginal distribution of $x$ in
the original input space.

\paragraph{Lower bounding the log marginal likelihood} 
Since the entropy of $e$ is constant wrt the model parameter $\pi$, $\hat{\pi}_\gA$ is equal to the maximizer of $\gL_\gA(\pi; x):=\E_{e}[\log p_\pi(x,e)] + H(e)$ averaged over all $x_i$'s.
For any $x\in\gX$, the quantity $\log p_\pi(x) - \gL_\gA(\pi; x) $ can be written as the KL divergence:
\begin{align*}
\log\,&p_\pi(x) - \gL_\gA(\pi; x) \\ 
&= \cancel{\log p_\pi(x)} - \E_{e}[\cancel{\log p_\pi(x)} + \log p_\pi(e|x)] - H(e) \\
&= \KL(q(e)||p_\pi(e|x))
\end{align*}

Since KL is non-negative, maximizing the joint likelihood according to Equation~\ref{eq:amle} is equivalent to maximizing a lower bound on the log marginal likelihood of $x$.
We refer to this as the \emph{Augmentation Gap}, as it reflects the incapability of the joint density to model the marginal of $e$ independently of $x$.

\begin{figure*}[th]
    \centering
    \includegraphics[height=6.0cm]{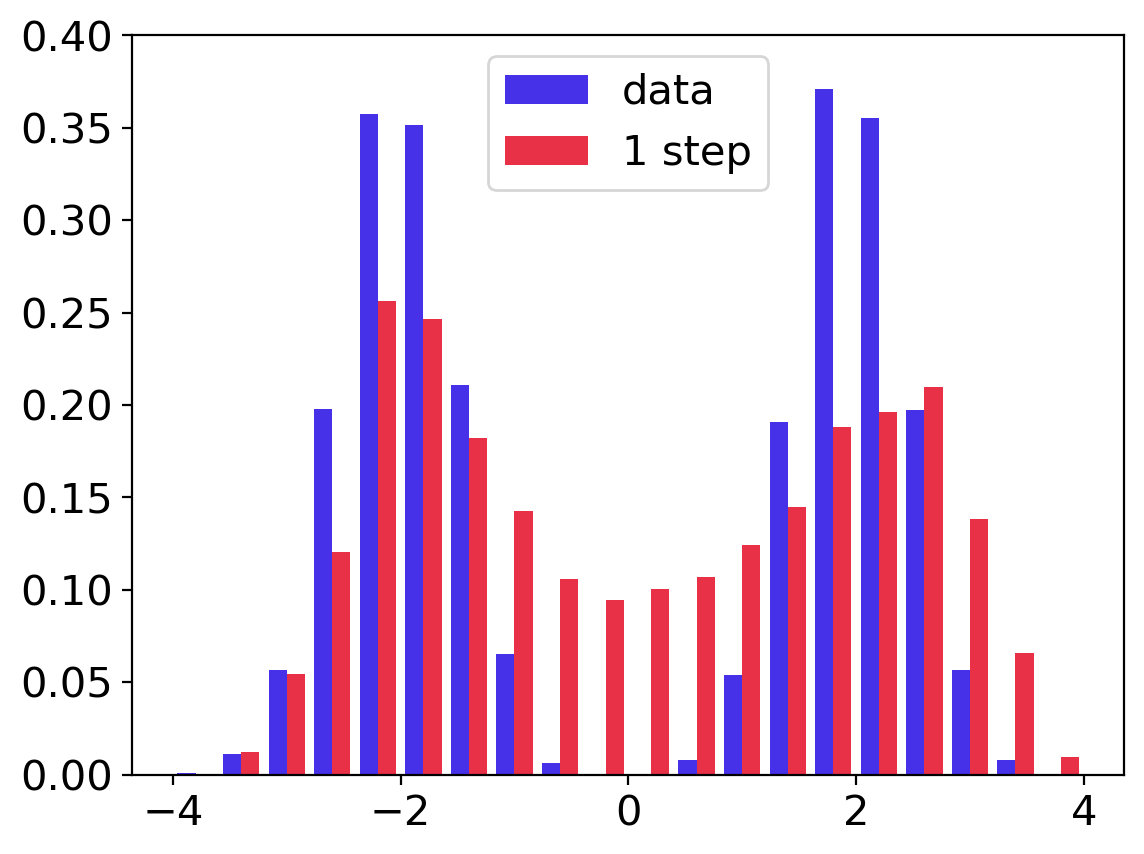}
    \hspace{0.6cm}
    \includegraphics[height=6.0cm]{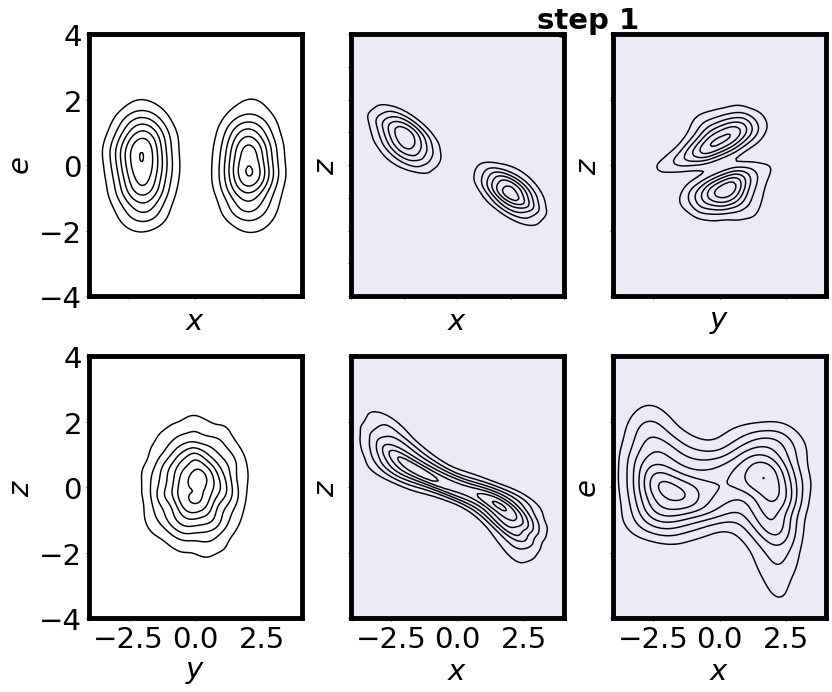}
    \vspace{-2mm}
    \caption{\small 
    Density modeling of 1D MoG with VAE (aka 1-step ANF). \emph{Left}: marginal distribution in the $\gX$-space. \emph{Right}: joint distribution in the $\gX\times\gE$-space. 
    The first row is the inference path, where the joint data density $q(x)q(e)$ is mapped by an encoding transform (transforming $e$ into $z$ conditioned on $x$) followed by a decoding transform (transforming $x$ into $y$ conditioned on $z$). 
    The second row is the generation path, where the joint prior density $p(y)p(z)$ is transformed by the inverse decoding (transforming $y$ into $x$) followed by the inverse encoding (transforming $z$ into $e$).}
    \label{fig:anf_1d_1}
\end{figure*}

\paragraph{Estimating the log marginal likelihood}
The log marginal likelihood $\log p_\pi(x)$ of the data can be estimated in a way similar to \citet{burda2015importance}, by drawing $K$ i.i.d. samples of $e_j\sim q(e)$ per $x$ to estimate the following stochastic lower bound:
$$\hat{\gL}_{\gA,K}(\pi)  
:= \log  \frac{1}{K}\sum_{j=1}^K \frac{p_\pi(x, e_j)}{q(e_j)}$$
which can be shown to be a consistent estimator for $\log p_\pi(x)$ and is monotonically tighter in expectation as we increase $K$.

\section{Augmented Normalizing Flows (ANF)}
\label{sec:anf}

We now demonstrate how to leverage the augmented input space to model the complex marginal distribution of the data. 
We consider maximizing the joint likelihood of $x$ coupled with a random noise $e\sim q(e)$.
Let $(y, z) \sim p(y,z)$ be drawn from some simple distribution, such as independent Gaussian. 
Assume the data $x, e$ is deterministically generated via an invertible mapping $x,e = F_\pi(y,z)$, with inverse $G_\pi=F_\pi^{-1}$. 
Then analogous to Equation~\ref{eq:cov}, $x, e$ has a joint density 
\begin{align*}
p_\pi(x, e)=\gN(G_\pi(x, e); 0,I)\left|\det\frac{\partial G_\pi(x, e)}{\partial (x, e)}\right|
\end{align*}

For simplicity, we can choose $q(e)$ to be the standard normal distribution. 
What we are left with is the choice of an invertible $G_\pi$ that can harness the augmented state space $\gE$ to induce a complex marginal on $\gX$. 
Inspired by the affine coupling proposed by \citet{dinh2016density}, we conditionally transform $x$ and $e$, hoping the structure in the marginal of $x$ can ``leak'' into $\gE$ and make the joint more easily Gaussianized. 
Concretely, we define two types of affine coupling
\begin{align*}
g_\pi^{\text{enc}}(x, e) &= \texttt{concat}(x,\, s_\pi^\text{enc}(x)\odot e + m_\pi^\text{enc}(x)), \\
g_\pi^{\text{dec}}(x, e) &= \texttt{concat}(s_\pi^\text{dec}(e)\odot x + m_\pi^\text{dec}(e),\, e)
\end{align*}
We refer to the pair of encoding transform and decoding transform as the \emph{autoencoding} transform. 
We stack them up in alternating order, i.e. 
$G_\pi = g_{\pi_N}^{\text{dec}}\circ g_{\pi_N}^{\text{enc}} \circ ... \circ g_{\pi_1}^{\text{dec}}\circ g_{\pi_1}^{\text{enc}}$ for $N\geq1$ steps, where $\pi=\{\pi_1,...,\pi_N\}$ is the set of all parameters.
See Figure \ref{fig:anfs}-(a,b) for an illustration.

\begin{figure*}
    \centering
    \includegraphics[width=1.00\textwidth,clip=true, trim=0 10px 0 5px]{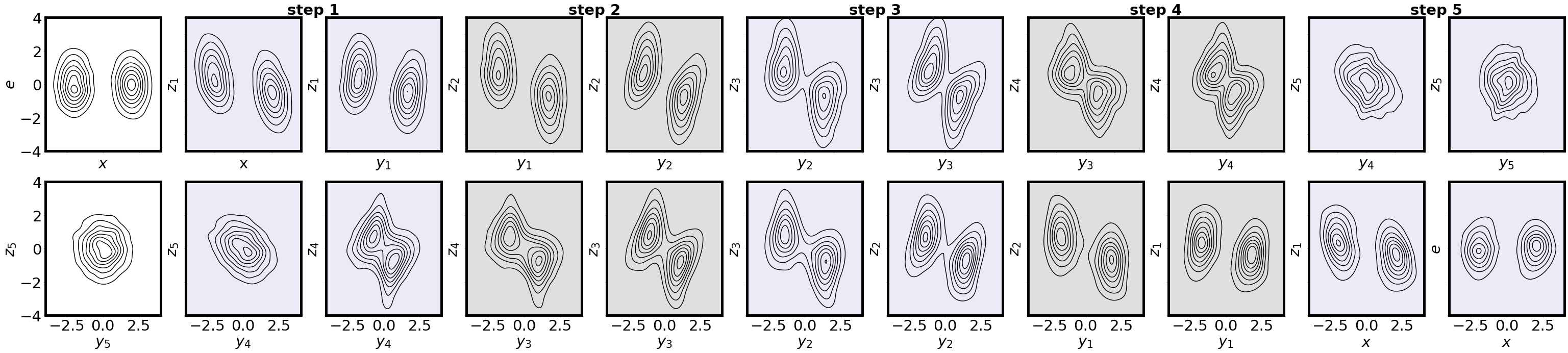}
    \caption{\small 5-step ANF on 1D MoG. In the inference path (top row), we start with an encoding transform that maps $e$ to $z_1$ conditioned on $x$, followed by a decoding transform that maps $x$ into $y_1$ conditioned on $z_1$. 
    We reuse the same encoder and decoder to refine the joint variable repeatedly to obtain $y_5$ and $z_5$. 
    In the generative path (bottom row), we reverse the process, starting with the inverse transform of the decoding, followed by the inverse transform of the encoding, etc. 
    }
    \label{fig:anf_1d_5}
\end{figure*}

\subsection{VAE as ANF}
Variational Autoencoders are a special case of augmented normalizing flows with only ``one step'' of encoding and decoding transform \citep{dinh2014nice}. 
To see this, assume the decoding distribution $p_\theta(x|z)$ is a factorized Gaussian with mean $\mu_\theta(z)$ and standard deviation $\sigma_\theta(z)$.
By letting $z=\mu_\phi(x)+\sigma_\phi(x)\cdot e$ and $y=(x-\mu_\theta(z))/\sigma_\theta(z)$ and applying the change of variable formula to both $q_\phi(z|x)$ and $p_\theta(x|z)$, we get from Equation (\ref{eq:elbo_rep}) 
\begingroup\makeatletter\def\f@size{8.501}\check@mathfonts
\def\maketag@@@#1{\hbox{\m@th\large\normalfont#1}}
\begin{flalign}
\gL(\theta,\phi; x) = 
\E_{e\sim q(e)}\Big[&\log \gN(y; 0,I) - \sum_i\log \sigma_{\theta,i}(z) \,+ \label{eq:elbo_cov} \\
&\log \gN(z; 0,I) + \sum_j\log \sigma_{\phi,j}(x)\Big] + H(e) \nonumber
\end{flalign}
\endgroup

Averaging over the data distribution $\hat{q}(x)$, we obtain the expected joint likelihood (up to the constant $H(e)$)
$$\E_{x,e\sim \hat{q}(x)q(e)}\left[\log\gN((y,z); 0,I)\left|\det\frac{\partial (y,z)}{\partial (x,e)}\right|\right]$$

The variational gap between the log marginal likelihood and the evidence lower bound is equal to the augmentation gap since the KL divergence is invariant under the transformation between $e\longleftrightarrow z$:
$$\KL(q(z|x)||p(z|x)) = \KL(q(e)||p(e|x))$$

\begin{figure}
    \centering
    \includegraphics[width=0.48\textwidth,clip=true, trim=100px 25px 0px 5px]{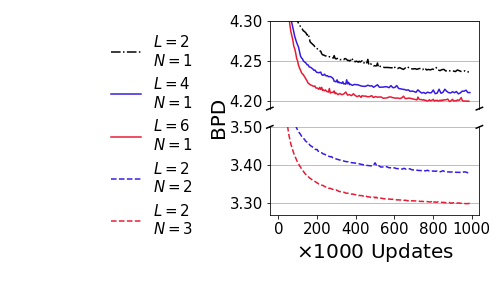}
    \vspace{-7mm}
    \caption{\small Comparing increasing number of layers of stochastic units ($L$) versus increasing number of layers of autoencoding transforms ($S$). (\emph{x-axis}): number of updates. (\emph{y-axis}): upper bound of bits per dim (BPD) on CIFAR 10 test data.
    }
    \label{fig:layers}
\end{figure}

This gives us an alternative interpretation of \emph{inference suboptimality}~\citep{cremer2018inference}: the inaccuracy of inferring the true posterior $p(z|x)$ can be attributed to the incapability of the joint density to model the augmented data $q(e)$. 

To illustrate this phenomenon, we model the density of a one dimensional mixture of Gaussian (1D MoG).
In Figure \ref{fig:anf_1d_1} (left), we plotted the density histograms of the MoG distribution (blue) and a one-step ANF, i.e. VAE with Gaussian encoder and decoder (orange), trained on the MoG samples. 
Not surprisingly, the latter fails to represent two well separated modes of probability mass. 
In Figure \ref{fig:anf_1d_1} (right), we visualize the joint density of the augmented data $x,e\sim q(x)q(e)$ throughout the transformation. 
We see that the transformed data $y,z=g^{dec}_{\pi_1}(g^{enc}_{\pi_1}(x,e))$ is not perfectly Gaussianized.
In fact, if we project it horizontally we can see that the ``aggregated posterior'' (marginal of $z$) does not match the prior distribution $p(z)$. 
As a result, the pushforward $x,e=g^{enc, -1}_{\pi_1}(g^{dec, -1}_{\pi_1}(y,z))$ of $y,z\sim p(y,z)$ does not follow the augmented data distribution $q(x)q(e)$ well. 
When we fix different values of $x$, we have different slices of density functions for $e$, indicating that $e$ and $x$ are dependent and that $p_\pi(e|x)$ deviates from $q(e)$. 

We carry out the same experiment on 1D MoG with multiple flow layers, which generalizes a VAE with Gaussian encoder and decoder. 
We set the number of flow layers (i.e. \emph{steps}) to be $5$.
To furthermore demonstrate the benefit of transformation composition, we also tie the parameters of each encoder and decoder step, separately.  
That is, the same set of parameters are used at different steps of encoding and decoding to make sure capacity stays constant. 
Since the conditional independence assumption in VAE is relaxed,
the augmented data is more successfully Gaussianized, as can be seen in Figure~\ref{fig:anf_1d_5}. The generated samples also follow the target joint density more closely.

\begin{table*}[ht]
\centering
\small
\begin{tabular}{lccccccccc} 
\toprule
{\bf Model} & {\bf MNIST} & {\bf CIFAR 10} & {\bf ImageNet 32} & {\bf ImageNet 64} & {\bf CelebA-HQ} \\ 
\midrule 
\textbf{\small Models with autoregressive components} \\
{VAE + IAF {\small \citep{kingma2016improved}}} & {--} & 3.11 & {--} & {--} & {--}\\
{PixelCNN {\small \citep{oord2016pixel}}} & {--} & 3.14 & {--} & {--} & {--}\\
{PixelCNN (multiscale) {\small \citep{reed2017parallel}}} & {--} & {--} & 3.95 & 3.70 & {--}\\
{PixelSNAIL {\small \citep{chen2017pixelsnail}}} & {--} & \textbf{2.85} & \textbf{3.80} & {--} & {--} \\
{SPN {\small \citep{menick2018generating}}} & {--} & {--} & \textbf{3.79} & \textbf{3.52} & \textbf{0.61}\\
\midrule
\textbf{\small Flow-based models} \\
{Real NVP {\small \citep{dinh2016density}}}  & 1.06 & 3.49 & 4.28 & 3.98 & {--} \\
{Glow {\small \citep{kingma2018glow}}}  &  1.05 & 3.35 & 4.09 & 3.81 & 1.03 \\
{FFJORD {\small \citep{grathwohl2018ffjord}}}  & 0.99 & 3.40 & {--} & {--} & {--} \\
{Residual {\small \citep{chen2019residualflows}}}  & 0.97 & 3.28 & 4.01 & 3.76 & 0.99 \\
{Flow++ {\small \citep{ho2019flow++}}}  & {--} & 3.09 & \textbf{3.86} & 3.69 & {--} \\
{MaCow {\small \citep{ma2019macow}}} & {--} & 3.16 & {--} & 3.69 & \textbf{0.67} \\
\midrule
{ANF (ours)} & \textbf{0.93} & \textbf{3.05} & 3.92 & \textbf{3.66} & 0.72\\
\bottomrule
\end{tabular}
\caption{\small Bits-per-dim estimates of standard benchmarks (the lower the better).
Results of Flow++, MaCow, and ANF are models that employ variational dequantization instead of uniform noise injection. Details can be found in the appendix. 
}
\label{tab:density}
\end{table*}

\subsection{Hierarchical Augmented Normalizing Flows}
The information flow of the encoding-decoding transform just described is limited to the size of the random vector $e$, which makes it hard to optimize for more realistic settings such as natural images. 
We thus propose a second architecture by following the hierarchical variational autoencoder, which is defined by two pairs of joint distributions~\footnote{This particular factorization of the variational distribution is known as the bottom-up inference. We leave the top-down inference~\citep{kingma2016improved} and the bidirectional inference~\citep{maaloe2019biva}, which benefit more from parameter sharing, for future work.} 
\begin{align*}
p(x,z_1,...,z_L)&=p(x|z_1,...,z_L)\prod_{l=1}^L p(z_l|z_{l+1},...,z_L) \\
q(z_1,...,z_L|x)&=\prod_{l=1}^L q(z_l|z_{1},...,z_{l-1}, x).
\end{align*}
When all the conditionals are Gaussian distributions, the corresponding ELBO can be similarly rearranged to be the loss function of an ANF (see Figure~\ref{fig:anfs}-(c)). 
The encoding transform for each $e_l$ is conditioned on the ``transformed'' preceding variables  $s_{\pi, l}^e(x,z_{<l})\odot e_l + m_{\pi, l}^e(x,z_{<l})$ due to the conditioning in $q(z_l|z_{<l},x)$.
The decoding transform on the other hand is conditioned on the ``original'' preceding variables $s_{\pi, l}^d(e_{>l})\odot e_l + m_{\pi, l}^d(e_{>l})$, which is block-wise inverse autoregressive~\citep{kingma2016improved}. 
When the conditioning mappings are convolutional, the lower level transformation preserves information of the input locally, which is then combined with the deterministic path of the decoding that ``sees'' more of the input. More details on the architecture are described in Appendix \ref{app:raeb}.

\section{ANFs as Approximate Hamiltonian ODE and Universality}
The affine-coupling autoencoding transform with augmented variable is reminiscent of the leap-frog integration of the Hamiltonian system \citep{neal2011mcmc}. 
More recently, it has been shown by \citet{taghvaei19a} that solving a family of Hamiltonian ordinary differential equations (ODE) with an infinite time horizon gives us a transport map from the initial (data) distribution to an arbitrary target distribution with a log-concave density function. 
This suggests we can develop an approximation theorem by using ANFs to approximately, numerically solve the ODE.

Formally, we define scaling coefficients $\alpha_t = \log \frac{2}{t}$ and $\beta_t = \gamma_t = \log t^2$.
Let $p(x)$ be the standard normal density, and $q(x)$ be the data distribution.
Let $q_0=q$ and $\Phi:\gX\rightarrow\R$ be some convex function. 
Define the Hamiltonian ODE:
\begin{align}
\dot{x}_t &= e^{\alpha_t-\gamma_t}e_t, && x_0\sim q_0 \nonumber 
\\
\dot{e}_t &= -e^{\alpha_t+\beta_t+\gamma_t}\nabla\log \frac{q_t(x_t)}{p(x_t)}, &&e_0=\nabla\Phi(x_0) \nonumber 
\end{align}
where $\dot{x}_t$ and $\dot{e}_t$ are the time derivatives of $x$ and $e$ at time $t$, and $q_t$ is the marginal density of $x_t$. 

Second, we construct a sequence of encoding and decoding functions $m^{\text{enc}}_{n}$ and $m^{\text{dec}}_{n}$ parameterized by neural networks, and define the following (additive) invertible mappings
\begin{align}
e^\pi_{1} &= e^\pi_0 + m^{\text{enc}}_{1}(x^\pi_0)\nonumber\\
x^\pi_{n+1} &= x^\pi_{n} + 2\e \cdot m^{\text{dec}}_{{n+1}}(e^\pi_{n+1})  && \forall\,n\geq0 \label{eq:x_nn_main}
\\
e^\pi_{n+1} &= e^\pi_n + 2\e \cdot m^{\text{enc}}_{{n+1}}(x^\pi_n) && \forall\,n\geq1 \label{eq:e_nn_main}
\end{align}
with $e^\pi_0=0$ and $x^\pi_0 \sim q_0$. 
The step size parameter $\e$ will be chosen to depend on the depth coefficient $N$, i.e. the number of steps of the joint transformation. 

Assume our target distribution lies within a family of distributions $\gQ$ satisfying Assumption \ref{assumption} in the Appendix \ref{app:proofs} (some smoothness condition on the time derivatives and $\Phi$). 
We can then set the encoding and decoding functions to be arbitrarily close to the time derivatives by the universal approximation of neural networks~\citep{cybenko1989approximation}, and by taking the depth $N$ to be arbitrarily large, we can approximate the transport map induced by the Hamiltonian ODE arbitrarily well, which gives rise to the following universal approximation theorem (the proof is relegated to the Appendix \ref{app:proofs}):
\begin{restatable}{thm}{anfdist}
\label{thm:anf_dist}
For any $q\in\gQ$, we can find a sequence $(x^\pi_N,e^\pi_N)$ of ANFs of the additive form (\ref{eq:x_nn_main},\ref{eq:e_nn_main}), such that if $x^\pi_0,e^\pi_0\sim q(x)\delta_0(e)$ and $x_\infty, e_\infty\sim p(x)\delta_0(e)$, then $(x^\pi_N,e^\pi_N)\rightarrow (x_\infty, e_\infty)$ in distribution. 
\end{restatable}

\begin{table*}[ht]
\centering
\resizebox{0.98\textwidth}{!}{%
\begin{tabular}{cc|cc|cccc|c|cc}
\toprule 
& & {\small PixelCNN${}^1$} & {\small PixelIQN${}^1$} & {\small i-ResNet${}^2$} & {\small Glow${}^2$} & {\small Residual Flow${}^2$} & {\small VAE+Glow${}^3$} & {\small ANF} & {\small DCGAN${}^4$} & {\small WGAN-GP (TTUR)${}^4$}
\\ 
\midrule
IS & ($\uparrow$) & 4.60 & 5.29 & {--} & {--} & {--} & {--} & \textbf{6.49} & 6.16 & 7.86 \\
FID & ($\downarrow$) & 65.93 & 49.46 & 65.01 & 46.90 & 46.37 & 42.14 & \textbf{30.60} & 37.7 & 29.3 (24.8) 
\\
\bottomrule
\end{tabular}
}
\caption{\small Evaluation on Inception Score (IS, the higher the better) and Fréchet Inception Distance (FID, the lower the better) of models trained on CIFAR 10. Results taken from ${}^1$\citet{ostrovski2018autoregressive}, ${}^2$\citet{chen2019residualflows}, ${}^3$\citet{vaeflow}, and ${}^4$\citet{gulrajani2017improved,heusel2017gans}. 
Parenthesis indicates two time-scale update rule for WGAN-GP.
}
\label{tab:scores}
\end{table*}
\begin{figure}[t!]
    \centering
    \includegraphics[width=0.48\textwidth]{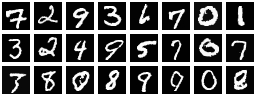}
    
    \vspace{1mm}
    
    \includegraphics[width=0.48\textwidth]{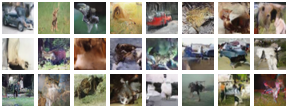}
    
    \vspace{1mm}
    
    \includegraphics[width=0.48\textwidth]{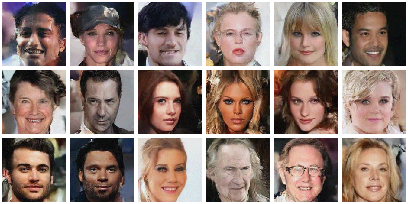}
    \vspace{-3mm}
    \caption{\small Unconditionally generated samples from models trained on MNIST (\emph{top row}), CIFAR 10 (\emph{middle row}), and 5-bit CelebA (\emph{bottom row}). 
    }
    \label{fig:samples}
\end{figure}

\section{Related Work}
In the literature of normalizing flows, much work has been done to improve expressivity while maintaining computational tractability. 
For example, \citet{dinh2014nice,dinh2016density} introduce an affine coupling that partitions the features into two groups so that the Jacobian is a block-triangular matrix. 
The resulting mapping is relatively restricted since it only models partial dependency. 
\citet{kingma2016improved} further exploits the ordered dependency by constructing an inverse autoregressive mapping but its inversion requires a computation time linear in dimensionality~\citep{papamakarios2017masked}, 
and does not even have a closed-form formula in the more general non-affine setting \citep{huang2018neural}.
\citet{behrmann2018invertible} propose a residual form of $f$ whose Jacobian log-determinant can be stochastically estimated~\citep{chen2019residualflows} but inversion is achieved iteratively, not in one pass. 

Normalizing flows have also been used as (1) an inference machine in the context of variational inference for continuous latent variable models~\citep{kingma2016improved,tomczak2016improving,berg2018sylvester}, and (2) a trainable component of the latent variable model \citep{chen2016variational,agrawal2016deep,huang2017learnable}.
ANFs lie at the intersection of normalizing flows and latent variable models when a specific type of block-conditioning transformation is applied, and allow us to 
unifyingly view
flow-based priors as marginal transformation in the space of $e$, and amortized flows for improving posterior inference as different variants of the encoding transform. 
Another way of improving the inference machine's expressivity is to consider a hierarchical model; in fact, ANFs can be viewed as a generalization of the auxiliary variable method for hierarchical variational inference \citep{agakov2004auxiliary,ranganath2016hierarchical}; see Appendix \ref{app:anf_vi} for the connection and \ref{app:related} for more discussion on future direction. 

Finally, \citet{dupont2019augmented} also employs augmentation to improve the expressivity and stability of a neural ODE \citep{chen2018neural}, and they believe such a method can be used to reduce the cost of training a continuous normalizing flow \citep{grathwohl2018ffjord}.

\begin{figure}
    \centering
    \includegraphics[width=0.44\textwidth]{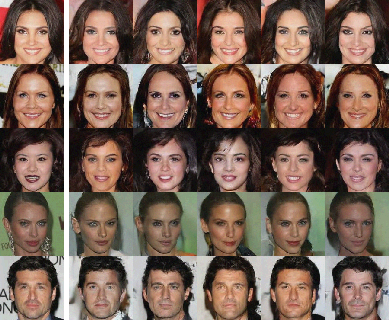}
    \caption{Lossy reconstruction. \emph{Left}: original data. \emph{Right}: reconstruction from the topmost representation.}
    \label{fig:recon}
\end{figure}

\begin{figure*}
    \centering
    \includegraphics[width=0.24\textwidth,clip=true, trim=0 4px 0 0px]{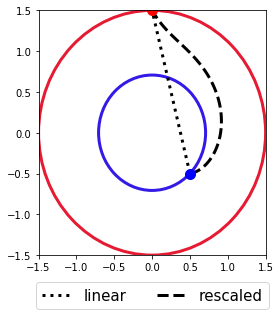}
    \hfill
    \includegraphics[width=0.725\textwidth,clip=true, trim=0 64px 0 0px]{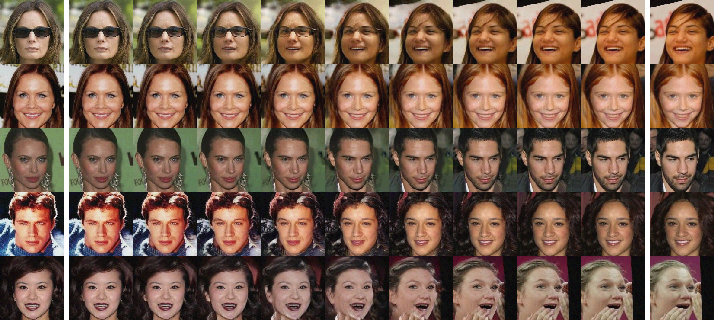}
    \caption{\emph{Left}: comparison of linear and rescaled interpolations. \emph{Right}: rescaled interpolation of input data (first and last columns).}
    \label{fig:interpolate}
\end{figure*}

\section{Large-Scale Experiments}

\subsection{Quantitative results}
In the more realistic settings, we augment the data with a hierarchy of noise, as described in the last part of Section \ref{sec:anf}. 
See Appendix \ref{app:exp} for more experimental details. 

\paragraph{Stochastic vs. deterministic features} We conduct an ablation study on the effect of composing multiple encoding-decoding transformations ($N$ steps) versus increasing the number of stochastic layers ($L$ layers). 
We monitor the bits per dim (BPD) of the test set of CIFAR 10~\citep{krizhevsky2009learning} throughout training.
Figure \ref{fig:layers} shows that increasing the number of flow layers can more effectively improve the likelihood of the model than increasing the number of stochastic layers.

\paragraph{Density estimation}
We perform density modelling on the MNIST handwritten digit dataset~\citep{lecun1998gradient}, CIFAR 10~\citep{krizhevsky2009learning}, downscaled versions of ImageNet ($32\times32$ and $64\times64$)~\citep{oord2016pixel} and the celebrity face dataset CelebA~\citep{liu2015faceattributes}, and compare with other state-of-the-art density models.  
In Table \ref{tab:density}, we see that ANFs set a few new records in terms of BPD on the standard benchmarks in the non-autoregressive category.
We use the importance sampling described in Section \ref{sec:amle} to estimate the log likelihood. 
The augmentation gap is around 0.01 BPD for all benchmarks, indicating the augmented flow is capable of achieving good likelihood estimate and high inference precision at the same time. 

\subsection{Qualitative results}

\paragraph{Sample quality} For quantitative evaluation of sample quality, we report the Inception Score (IS) \citep{salimans2016improved} and the Fréchet Inception Distance (FID) \citep{heusel2017gans}, expanding the table of \citet{ostrovski2018autoregressive}. 
We found the FID score of WGAN-GP~\citep{gulrajani2017improved} reported in \citet{ostrovski2018autoregressive} is worse than the one reported in the literature, so we include the original values of IS and FID of GANs from the original works of \citet{gulrajani2017improved} and \citet{heusel2017gans} for more realistic comparison.
In Table \ref{tab:scores}, we see that ANF obtains better scores than all the other explicit density models, and is close to matching the FID of the orignal WGAN-GP by~\citet{gulrajani2017improved}. 
The generated samples are presented in Figure~\ref{fig:samples} and Appendix \ref{app:samples}.
Since the encoding-decoding transformation has a receptive field that is wide enough to cover the entire raw data, the generated samples also look more globally coherent. 

\paragraph{Lossy reconstruction}
As a hierarchical model, ANF can be used to perform inference for the higher level representation, and sample the lower level details for reconstruction. 
We do this by sampling $e_1,...,e_L$, obtaining the corresponding $y,z_1,...,z_L \leftarrow G_\pi(x, e_1, ..., e_L)$, randomizing all but the last representations $y', z_1',..., z_{L-1}' \sim \gN(0,I)$, and reconstructing from the new joint representation $x', e_1', ..., e_L' \leftarrow G_\pi^{-1}(y',z_1', ..., z_{L-1}', z_L)$. 
Similar to other hierarchical models~\citep{gulrajani2017pixelvae, belghazi2018hierarchical}, 
ANF is also capable of retaining global, semantic information of the raw data stored in its higher level code; this is shown in Figure~\ref{fig:recon}.

\paragraph{Interpolation}
We also perform interpolation in the latent space between real images. 
Previous works such as \citet{kingma2018glow} perform linear interpolation of the form $h(u, v, t)=tu+(1-t)v$ for $t\in[0,1]$, which we observe has a non-smooth transition (e.g. sudden color change). 
We hypothesize this is due to the fact that convex combination of two vectors would result in an increase and then a decrease in the density of the standard Gaussian prior. 
This is undesirable since the interpolated points are atypical because Gaussian samples are known to concentrate around the shell (of radius proportional to square root dimensionality). 
Hence, we propose the \emph{rescaled interpolation}
\begin{align}
h'(u,v,t) = \frac{h(||u||,||v||,t)}{||h(u,v,t)||}\cdot h(u,v,t)
\end{align}
where $||\cdot||$ denotes the L2 norm, to make sure the scale of the resulting point is a linear interpolation of the scales of the two input vectors.
The result in Figure~\ref{fig:interpolate} shows that the transition is extremely smooth (see Appendix \ref{app:inter} for a side-by-side comparison with linear interpolation) and the intermediate images are realistic looking.

\section{Conclusion}
In this work, we propose the \emph{Augmented Normalizing Flows} and a corresponding variational lower bound on the marginal likelihood.
We show that the proposed method can be used to approximate a Hamiltonian dynamical system as a universal transport map and achieves competitive or better results than state-of-the-art flow-based methods.

\iftrue
\section*{Acknowledgements}{
CW would like to thank 
Matt Hoffman for a discussion on deterministic Langevin transitions and Amirhossein Taghvaei for referencing the work of Wang \& Li. 
Special thanks to people who have provided their feedback and advice during discussion or while reviewing the manuscript, including Valentin Thomas, Joey Bose, and Eeshan Dhekane; to
Taoli Cheng and Bogdan Mazoure for volunteering for the internal review at Mila; and to Ahmed Touati, Christos Tsirigotis and Jose Gallego for proofreading parts of the proof. 
}
\fi

\bibliography{reference}
\bibliographystyle{icml2020}

\clearpage
\appendix
\onecolumn

\section{Interpolation}
\label{app:inter}
We compare linear interpolation with rescaled interpolation (rescaling is done separately for each stochastic layer). 
We see that the middle points of linear interpolation tend to be more yellowish, and rescaled interpolation results in a smoother and direct transition between two input vectors.
\begin{figure*}[h!]
    \centering
    \includegraphics[width=0.49\textwidth]{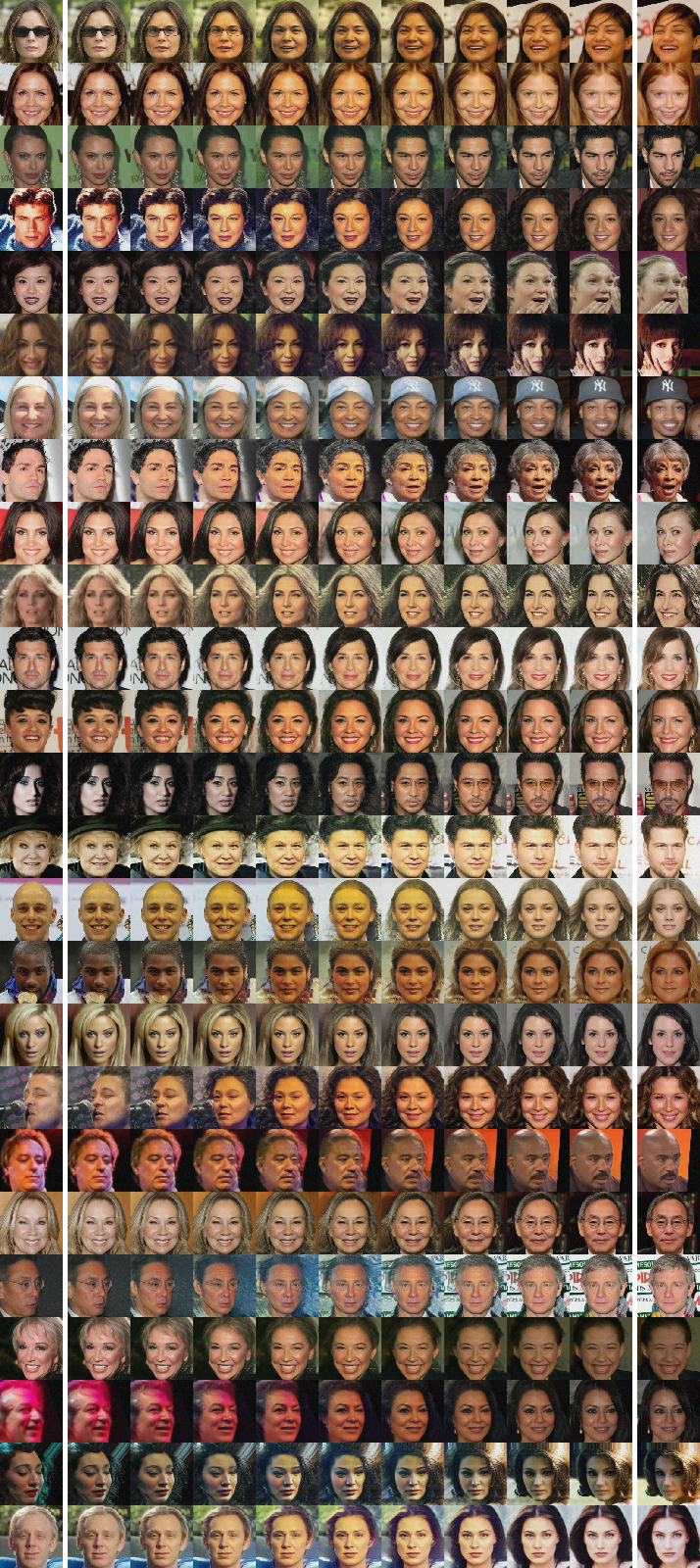}
    \hfill
    \includegraphics[width=0.49\textwidth]{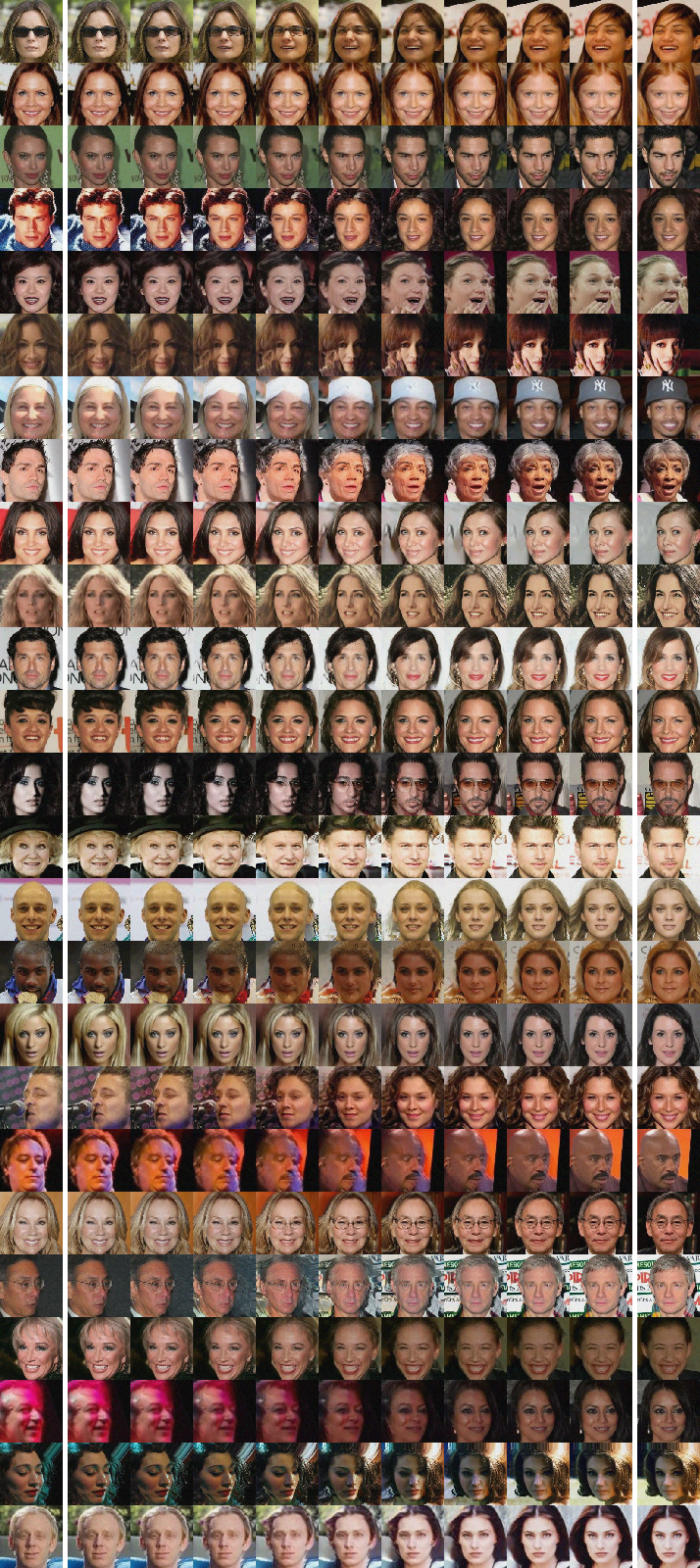}
    \caption{(\emph{left}) linear interpolation, (\emph{right}) rescaled interpolation}
    \label{fig:interpolate_compare}
\end{figure*}


\clearpage

\section{Experiment Details}
\label{app:exp}
To model natural images, we employ a more intricate architecture with a higher modeling capacity described in \ref{app:raeb}. 
Section \ref{app:init} describes the parameter initialization scheme and parameterization constraints that are imposed to stablize training. 
In \ref{app:warm}, we propose an objective-annealing technique that biases the autoencoding transform to 
focus on Gaussianizing the raw data more at the early stage of training.
We found this to be helpful for optimization. 
All the hyperparameters used in the experiments are summarized in \ref{app:hparam}.

\subsection{Residual autoencoding blocks}
\label{app:raeb}
\begin{figure*}[h!]
	\centering
	\includegraphics[width=0.95\textwidth]{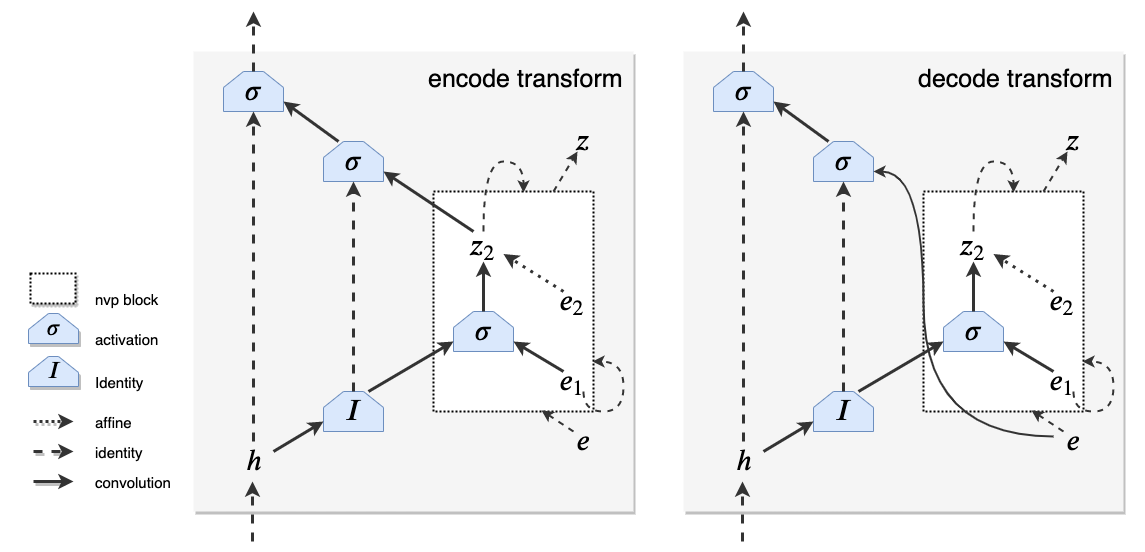}
	\caption{Architecture of a single encode transform block and a single decode transform block.}
	\label{fig:autoencoding_block}
\end{figure*}

We parameterize all of the feedforward layers with weight normalization \citep{salimans2016weight}. 
For the encoding transforms, we first map $x$ using a convolutional layer composed with an activation function, followed by a pooling layer to obtain the first set of conditional features $h$. 
We now use these features to conditionally transform each of the stochastic units in the order $e_1,...,e_L$, using an \emph{encoding block}.
The encoding block outputs a set of modified conditional features, which is then used to modify the next $e_l$ (except when the feature map size is halved, in which case an additional pair of convolutional layer and pooling layer is first applied).

\paragraph{Encode transform and decode transform} 
Each encoding block has a nested residual structure, taking in the conditional features as input to transform the corresponding stochastic unit $e$, illustrated in Figure~\ref{fig:autoencoding_block}. 
The conditional features are first convolved and then fed into a nested Real NVP block.   
The output of the Real NVP block (applied twice, see below) is then convolved and added to the convolved conditional features.
We apply non-linearity before convolving again and adding to the the original conditional features via skip connection.

The decode transform is similar, except we convolve the incoming stochastic unit to modify the conditional features, thus having a shorter computational path to reconstruct the data.

\paragraph{Real NVP block}
The stochastic unit is split into two halves ($e_1$ and $e_2$) using the checkerboard mask. 
The convolved conditional feature is concatenated with the masked stochastic unit (the part that is not masked is denoted as $e_1$) to transform the part of the stochastic unit being masked out ($e_2$). 
The same Real NVP block is reused (using the same set of parameters), alternating the pattern of the mask to transform the other half of the stochastic unit (with $e_1$ and $e_2$ swapped). 
We found sharing the parameters of these two consecutive Real NVP transformations to improve the convergence of the likelihood.

\subsection{Variational dequantization}
We use the variational dequantization proposed by ~\citet{ho2019flow++} in all our density estimation experiments. 
We first map the input image $x$ into a deterministic feature $x'$ space using a convolutional network of the following form: 
$$\texttt{conv(stride=2) -> act -> conv -> act -> bilinear\_upsample(stride=2)}$$
where \texttt{act} denotes activation function. 
Note that the input to this convolution network is rescaled to $[0,1]$ via $x/(2^\texttt{n\_bits}-1)$, where \texttt{n\_bits} is the number of bits.
We then use the Real NVP block to transform a standard Gaussian noise, where $x'$ acts as the conditional feature. 
We apply two Real NVP blocks to obtain $u\_logit$, and transform it into $u$ using the logistic sigmoid activation function so that each element of $u$ lies within $(0,1)$.
We then perturb the data by $(x+u)/2^{\texttt{n\_bits}}$, which is then passed through a \texttt{clip} operator for numerical stability. 
We define \texttt{clip} to be 
$$\texttt{clip}(x, \delta) = x\cdot(1-\delta) + 0.5\cdot\delta$$
where $\delta=0.1$. 
Finally, we pass the clipped value into the logit function (inverse sigmoid) to obtain the dequantized data. 
We have taken into account the probability density of $u\_logit$ and all the changes of variable (i.e. sigmoid, rescaling, clipping, and logit) when computing the lower bound.

\subsection{Initialization and parameterization constraint}
\label{app:init}
Unless it is otherwise stated, we initialize all the convolutional kernels using truncated normal distribution with standard deviation $0.1$, and for weight normalization, the rescaling parameter is set to be $1.0$ and the shifting parameter $0.0$. 
Only for the last layer of the Real NVP block we initialize $g$ to be $0.0$. 
We apply a split operator to this last layer to obtain a ``shift'' coefficient and ``log scale'' coefficient for affine transformation. 
The last layer has double the dimensionality of the stochastic unit to be transformed.
The split operator simply splits it in two parts. 
For the log scale coefficient, we apply the log-sigmoid function to make sure after exponentiation, it is bounded between $0$ and $1$ (similar to \citet{kingma2018glow}). 
Since the pre-log-sigmoid is initialized to be $0$, we add in a constant that depends on the total number of transformations that will be apply to the stochastic unit such that the overall transformation (after composition) will rescale the raw input by a factor of $0.95$ (without considering activation normalization). 
This is to ensure the entire transformation is more robust to variation of depths at initialization. 

We also apply activation normalization \citep{kingma2018glow} with data-dependent initialization that standardizes the transformed feature, after each encoding transform and each decoding transform. 
We clip the log scale coefficient at $\pm 2.5$.

\subsection{Deterministic warm up}
\label{app:warm}
Due to our choice of flow, our instantiation of ANF resembles a VAE. 
It has been previously shown that starting off with less regularization and noise injection is beneficial to training, a technique known as deterministic warm up \citep{raiko2007building,sonderby2016ladder}. 
Similarly, if we expand the objective of ANF with autoencoding transform (affine coupling), we get
\begin{align*}
\log p(x,e) = \log\gN(y_T; 0,I) + \log\gN(z_T; 0,I) + \sum_{t=1}^T \sum_{j=1}^{d} \log s_{\pi_t}^\text{enc}(y_{t-1})_j + \sum_{t=1}^T \sum_{j=1}^{d} \log s_{\pi_t}^\text{dec}(z_{t})_j 
\end{align*}
where $d$ is the dimensionality of the augmented data $e$, and $(y_t, z_t)$ are defined recursively as 
\begin{align*}
z_t &= s_{\pi_t}^\text{enc}(y_{t-1})\odot z_{t-1} + m_{\pi_t}^\text{enc}(y_{t-1}) \\
y_t &= s_{\pi_t}^\text{dec}(z_t)\odot y_{t-1} + m_{\pi_t}^\text{dec}(z_t)
\end{align*} 
with the initial values $z_0=e$, $y_0=x$.
We modify the objective by lowering the weighting of $s_{\pi_t}^\text{enc}$ and $\log\gN(z_T; 0,I)$ such that the network can focus more on Gaussianizing the raw input $x$. 
We defined the modified objective as
$$\gL(\pi; x,e,\beta):= 
\log\gN(y_T; 0,I) + \sum_{t=1}^T \sum_{j=1}^{d} \log s_{\pi_t}^\text{enc}(y_{t-1})_j + \beta\left(\log\gN(z_T; 0,I) + \sum_{t=1}^T \sum_{j=1}^{d} \log s_{\pi_t}^\text{dec}(z_{t})_j\right)$$
where $\pi$ is all the trainable parameters.
We linearly anneal the weighting coefficient $\beta$ from $0$ to $1$ for the first $\alpha$ iterations of the training.
Note that in practice we apply the same $\beta$ to all augmented data in the hierarchical setup.

\newpage
\subsection{Hyperparameters}
\label{app:hparam}
{Notation summary for hyperparameters:}
\vspace{-3mm}
\begin{itemize}\setlength\itemsep{0.0em}
\item $L$: number of stochastic units ($e_1,...,e_L$).
\item $N$: number of steps (encoding-decoding pairs).
\item $K$: number of samples for importance sampling.
\item $\lambda$: decoupled weight decay coefficient for Adam.
\item $c$: number of channels (all deterministic features). 
\item $c'$: number of channels for the $l$'th stochastic unit (stochastic features). 
Power denotes repetition.   
\item $\mu$: feature map size (squared).  
\item $k$: kernel size (except for the data space layer). 
\item $b$: batch size. 
\item $s$: step size.
\item $a$: annealing schedule (number of parameter updates).
\item $u$: number of updates.
\item $\sigma$: activation function
\end{itemize}

\begin{table*}[ht]
\centering
\begin{tabular}{cccccc} 
\toprule
{\bf } & {\bf MNIST} & {\bf CIFAR 10} & {\bf ImageNet 32} & {\bf ImageNet 64} & {\bf CelebA-HQ} \\ 
\midrule
$L$ & 4 & 8 & 8 & 6 & 5\\
$N$ & 5 & 5 & 5 & 5 & 5\\
$K$ & 5000 & 1000 & 1000 & 1000 & 1000\\
$\lambda$ & 1e-5 & 1e-5 & 0 & 0 & 0\\
$c$ & 64 & 160 & 256 & 256 & 160\\
$c'$ & 2,2,2,2 & 32,28,...,4 & 32,28,...,4 & 24,20,...,4 & 20,16,...,4\\
$\mu$ & 14\textsuperscript{2},7\textsuperscript{2} & 16\textsuperscript{4},8\textsuperscript{4} & 16\textsuperscript{4},8\textsuperscript{4} & 32\textsuperscript{2},16\textsuperscript{2},8\textsuperscript{2} & 128,64,32,16,8 \\
$k$ & 3 & 3 & 3 & 3 & 3\\
$b$ & 64 & 64 & 64 & 64 & 12\\
$s$ & 0.001 & 0.001 & 0.001 & 0.0005 & 0.0005\\
$a$ & 20K & 20K & 20K & 20K & 20K\\
$u$ & 1M & 1M & 2M & 2M & 2M\\
$\sigma$ & Swish & Swish & Swish & Swish & Swish\\
\bottomrule
\end{tabular}
\caption{Hyperparameter details of density estimation tasks.}
\label{tab:exp_details}
\end{table*}

\clearpage

\section{Extended related work and future direction}
\label{app:related}
\paragraph{Normalizing flows.}
The term \emph{Normalizing Flow} was originally coined by \citet{tabak2010density,tabak2013family} where it was used for density estimation.
Differentiable bijective models were first introduced to the deep learning community as likelihood-based generative models by \citet{rippel2013high,dinh2014nice}, and as an inference machine by \citet{rezende2015variational}. 
Most development within this line of research is dedicated to improving the expressivity of the bijective mapping while maintaining computational tractability of the log-determinant of the Jacobian. 
Each family of flows can be characterized by the ``trick'' used to achieve this, e.g.
\begin{itemize}
\item \emph{Partial ordered dependency}. 
If the mapping has a partial and ordered dependency, its Jacobian matrix will be a triangular matrix, the determinant of which can be computed in linear time. This includes the following:
\begin{itemize}
    \item \citet{dinh2014nice,dinh2016density,kingma2018glow,ho2019flow++} use a block-wise conditioning in the mapping, and
    \item \citet{kingma2016improved,chen2016variational, papamakarios2017masked,huang2018neural} generalize block-wise dependency to temporal dependency wherein all the variables prior to the current variable of a given ordering are inputs of the conditioning to transform the current variable.
\end{itemize}
\item \emph{Low rank transform}.
If the mapping is of a particular residual form, the Jacobian determinant can be computed readily using the matrix determinant lemma \citep{rezende2015variational}
or its higher rank generalization \citep{berg2018sylvester}.
\item \emph{Lipschitz residual flow}. If the nonlinear block of a residual mapping is no more than 1-Lipschitz, the overall mapping is invertible, and its Jacobian can be estimated using power series expansion, the Hutchinsons trace estimator \citep{behrmann2018invertible} and the Russian roulette estimator \citep{chen2019residualflows}. 
\item \emph{Special convolutional forms}. Certain structure of convolutional kernels can also be designed to to ensure tractability, such as via using $1\times1$ convolution \citep{kingma2018glow}, masking \citep{oord2016pixel,NIPS2016_6527,hoogeboom2019emerging,song2019mintnet,ma2019macow} or imposing certain repeated structure \citep{karami2019invertible}.
\end{itemize}

In this work, we introduce the augmentation trick, which generalizes flow-based methods in an orthogonal yet complementary manner. 
In particular, we employ the coupling proposed by \citet{dinh2016density} to transform the augmented data; one potential alternative is to replace it with any of the tricks mentioned above. 

\paragraph{Architectures and parameter sharing.} As the autoencoding transform we use generalizes VAEs and hierarchical VAEs, one potential direction is to consider parameterizations that have shared components which are shown to be conducive to training, such as the ResNet with top-downn inference \citep{kingma2016improved} and the bidirectional inference machine \citep{maaloe2019biva}.
As a generalization of VAEs, ANFs can also be applied to latent variable models of different graphical representations, such as variational recurrent neural networks \citep{chung2015recurrent} and models of sets \citep{edwards2016towards}; for example, the \emph{set flow} proposed by \citet{rasul2019set} is an instance of permutation-invariant ANF applied to sets. 
Another avenue for improving parameter efficiency is to consider tying the weights of different steps of transformations. 
As our theory suggests, consecutive transformations of the discretized Hamiltonian ODE would differ only slightly if the time derivatives are smooth enough. 
This means it would be sufficient to consider a single network which also takes in time embedding as input for all transformations.

\paragraph{Approximate Hamiltonian flows.}
Our approximation theory builds on the result of \citet{wang2019accelerated}, which follows the optimal control framework of \citet{wibisono2016variational}.
The augmented variable is treated as the costate, which is deterministically dependent on the state, i.e. the input data. 
Therefore we set the initial augmented distribution to be a Dirac point mass for the approximation theory to hold. 
Our theorem can be improved if one can show some time trajectories with 
the augmented variable drawn independently from a non-degenerate initial distribution are convergent to the prior distribution. 
We leave this for future work. 
Meanwhile, the same proof technique can be used to study the approximation capability of different families of flows.
In particular, the residual flows \citep{behrmann2018invertible,chen2019residualflows} and their continuous counterpart \citep{chen2018neural,grathwohl2018ffjord} can be used to approximate the deterministic Langevin diffusion, since
(1) one can replace the Brownian motion term with the gradient of the log marginal density without modifying its Fokker-Planck equation (see \citet{hoffman2019langevin} or the appendix of \citet{wang2019accelerated}) and (2) the first-order Langevin dynamic is known to be convergent to its stationary distribution \citep{roberts1996exponential}. 

\paragraph{Gradient-based flows.}
As the theory suggests, gradient of the potential can be used to guide the evolution of the particle.
This has been previously explored by \citet{duvenaud2016early}. 
\citet{salimans2015markov} on the other hand propose a hierarchical model inspired by the Hamiltonian dynamic, and \citet{song2017nice,levy2017generalizing} generalize Hamiltonian Monte Carlo (HMC) with trainable neural components. 
Similarly, one can parameterize a Gradient-based augmented generative flow to model the data distribution.

\paragraph{Normalizing flows for variational inference.}
The most well-known application of normalizing flows is to improve the variational distribution to approximate posterior distribution of (1) the latent representations \citep{rezende2015variational, kingma2016improved,tomczak2016improving,berg2018sylvester} and (2) the parameters of neural networks \citep{louizos2017multiplicative, krueger2017bayesian, huang2019stochastic}. 
ANF can also be applied to inference problems, with slight modification of the target potential. 
We show in Appendix \ref{app:anf_vi} that one can augment the target distribution with an independent distribution and infer the joint target altogether.
This boils down to the hierarchical variational method \citep{agakov2004auxiliary,ranganath2016hierarchical} as a special case when one step of autoencoding transform is applied. 

\paragraph{Variational gap.}
The joint likelihood that we maximize is a variational objective lower-bounding the marginal likelihood of the data. 
One potential avenue for improvement is to reduce this bias (the augmentation gap) throughout training, by closing up the gap via importance sampling \citep{burda2015importance} or using an unbiased estimate of the marginal likelihood \citep{luo2020sumo}.

\paragraph{Representation learning.}
Considering invertible transformations in an augmented data space allows us to sidestep the topology-preserving property of a homeomorphism. 
The issue of this property is discussed and addressed by \citet{cornish2019localised} by converting the flow into a latent-variable model. 
\citet{dupont2019augmented} adopt the same technique by augmenting the data space and apply the augmented continuous time flow to discriminative tasks. 
We hypothesize this can potentially improve the representation learned by an invertible model, for example in a semi-supervised setting \citep{nalisnick2019hybrid,atanov2019semi} or as a component of a reversible model for memory-efficient backpropagation \citep{gomez2017reversible}.

\clearpage

\section{Augmented Normalizing Flows for Variational Inference}
\label{app:anf_vi}
Augmented normalizing flows can also be used for inference tasks where our goal is to approximate an unnormalized density $\tilde{p}(z)$ with a parametric distribution $q(z)$. 
This includes variational training of energy based models \citep{dai2017calibrating, zhai2016generative}, entropy regularized policy gradient in reinforcement learning \citep{mazoure2019leveraging, ward2019improving}, probability distillation \citep{oord2017parallel}, and variational Bayesian inference of latent variables \citep{kingma2013auto}. 

We focus on the case of variational inference (but the same technique can be used for other applications), where $\tilde{p(z)}=p(x,z)$, and our goal is to maximize the ELBO
$$\E_{z}\left[\log\frac{p(x,z)}{q(z)}\right]$$
where we can apply the standard change of variable to get $q(z)=q(e)\left|\frac{\partial g(e)}{\partial e}\right|^{-1}$ with $z=g(e)$ as described in Section~\ref{sec:background}.
Alternatively, we can augment the target distribution $\tilde{p}(z)$ with an independent $p(v)$, and jointly transform a base distribution $q(e)q(u)$ into $q(z,v)$ to approximate $\tilde{p}(z)p(v)$ via an invertible map $e,u\mapsto G(e, u)$.
Concretely, we maximize the following quantity
\begin{align}
\E_{z,v}\left[\log\frac{p(x,z)p(v)}{q(z,v)}\right]
=\E_{e,u}\left[\log\frac{p(x,G(e, u)|_1)p(G(e, u)|_2)}{q(e,u)}\left|\frac{\partial G(e,u)}{\partial (e,u)}\right|\right]
\label{eq:elbo_anf}
\end{align}
where $|_1$ and $|_2$ denote the first and the second coordinates, respectively.
This lower-bounds the ELBO since 
$$\E_{z}\left[\log\frac{p(x,z)}{q(z)}\right] - \E_{z,v}\left[\log\frac{p(x,z)p(v)}{q(z,v)}\right] = \E_{z,v}\left[\log \frac{q(v|z)}{p(v)}\right] = \E_z[\KL(q(v|z)||p(v))]$$
is non-negative. 

{\bf Auxiliary variable for hierarchical variational inference.} The above derivation for applying ANF to variational inference is reminiscent of the \textit{auxiliary variable method} \citep{agakov2004auxiliary, ranganath2016hierarchical}.
To see this, assume we parameterize $G(e,u)$ as the composition $g^\text{enc}\circ g^\text{dec}$, where 
\begin{align*}
g^{\text{enc}}(e, u) &= \texttt{concat}(e,\, s^\text{enc}(e)\odot u + m^\text{enc}(e)), \\
g^{\text{dec}}(e, u) &= \texttt{concat}(s^\text{dec}(u)\odot e + m^\text{dec}(u),\, u)
\end{align*}
with $s^\text{enc}, s^\text{dec}>0$.
Then Equation (\ref{eq:elbo_anf})
becomes
\begin{align*}
\E_{e,u}\left[\log\frac{p(x, s^\text{dec}(u)\odot e + m^\text{dec}(u))p( s^\text{enc}(z)\odot u + m^\text{enc}(z) )}{q(e,u)}+\log\sum_i s^\text{dec}(u)_i + \log\sum_j s^\text{enc}(z)_j\right]
\end{align*}
where $z:= s^\text{dec}(u)\odot e + m^\text{dec}(u)$, which is equivalent to
\begin{align*}
\E_{z,u}\left[\log\frac{p(x,z)\gN(u ;  - m^\text{enc}(e)/s^\text{enc}(e), s^\text{enc}(e)^{-2})}{\gN(z; m^\text{dec}(u), s^\text{dec}(u)^2)q(u)} \right] = \E_{z,u}\left[\log\frac{p(x,z) r(u|z)}{q(z|u)q(u)} \right] 
\end{align*}
where $q(z|u)=\gN(z; m^\text{dec}(u), s^\text{dec}(u)^2)$ and $r(u|z)=\gN(u ;  - m^\text{enc}(e)/s^\text{enc}(e), s^\text{enc}(e)^{-2})$.
This shows hierarchical variational methods are a special case of ANF, and the latter can potentially be used to improve the joint expressivity of the former through additional composition.

\newpage
\section{More samples}
\label{app:samples}
\subsection{CIFAR 10}
\begin{figure*}[h!]
    \centering
    \includegraphics[width=0.95\textwidth]{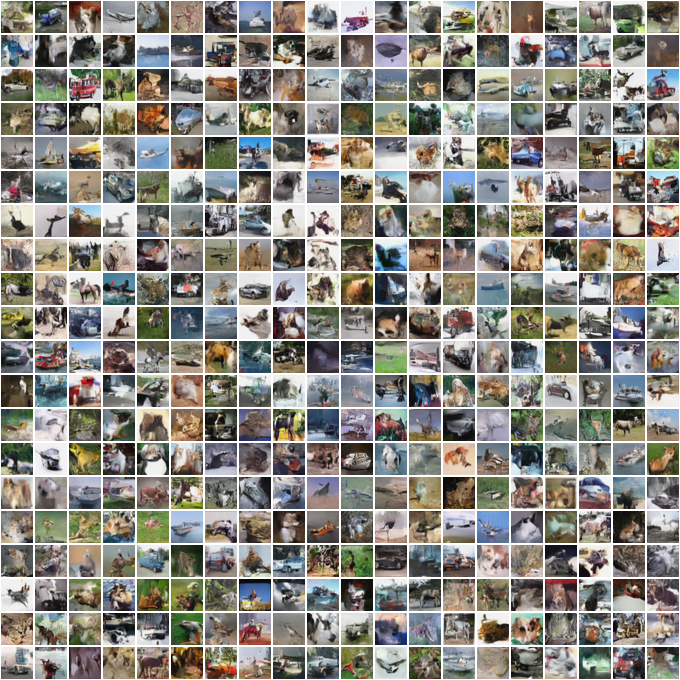}
    \caption{CIFAR 10 samples}
    \label{fig:cifar_grid}
\end{figure*}

\newpage
\subsection{Celeba 64}
\begin{figure*}[h!]
    \centering
    \includegraphics[width=0.95\textwidth]{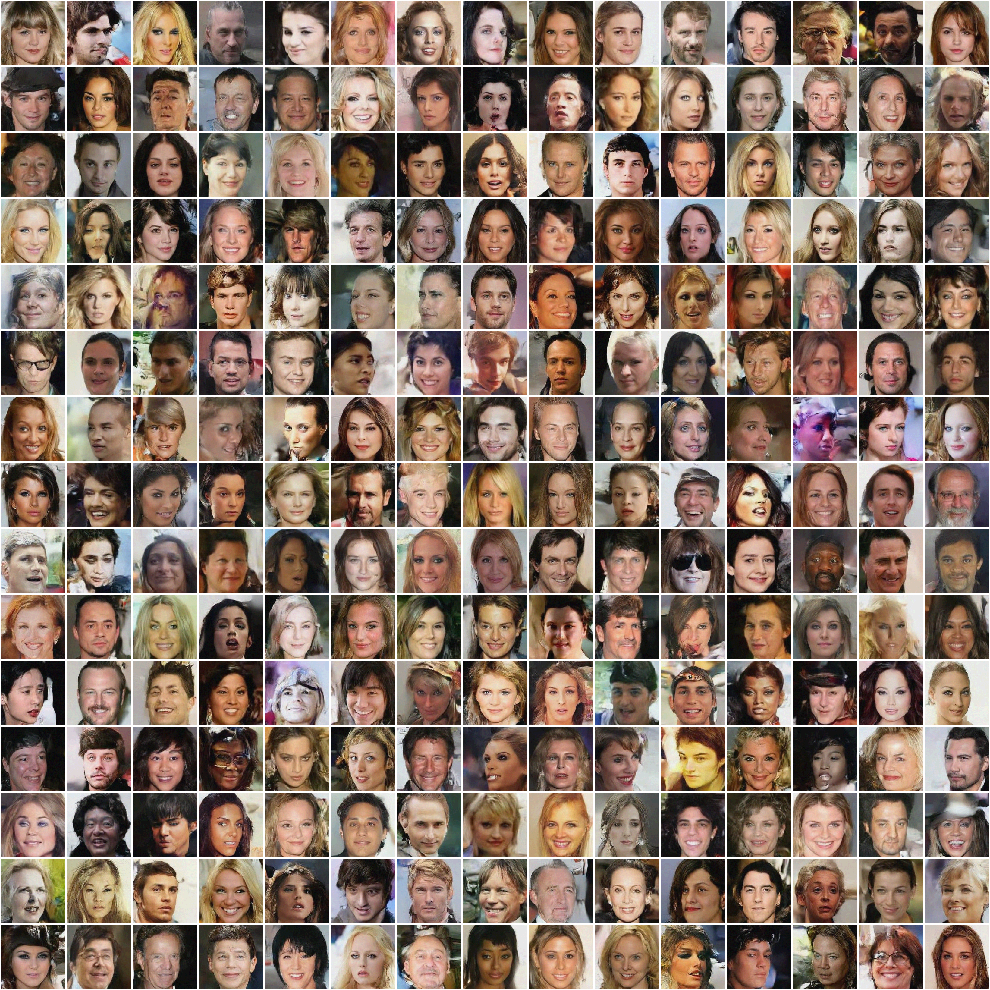}
    \caption{5-bit CelebA 64 samples}
    \label{fig:cifar_grid}
\end{figure*}

\newpage
\section{Proofs}
\label{app:proofs}

Define the scaling coefficients $\alpha_t = \log \frac{2}{t}$ and $\beta_t = \gamma_t = \log t^2$.
Let $p(x)$ be the standard normal density, and $q(x)$ be the data distribution.
Let $q_0=q$ and $\Phi:\gX\rightarrow\R$ be some continuous function. 
Define the following Hamiltonian ordinary differential equation (ODE):
\begin{align}
\dot{x}_t &= e^{\alpha_t-\gamma_t}e_t, && x_0\sim q_0 \label{eq:time_der_x}\\
\dot{e}_t &= -e^{\alpha_t+\beta_t+\gamma_t}\nabla\log \frac{q_t(x_t)}{p(x_t)}, &&e_0=\nabla\Phi(x_0) \label{eq:time_der_e}
\end{align}
where $\dot{x}_t$ and $\dot{e}_t$ are the time derivatives of $x$ and $e$ at time $t$, and $q_t$ is the marginal density of $x_t$. 

\begin{proposition}
For some convex $\Phi$, the trajectories of $x_t$ and $e_t$ following (\ref{eq:time_der_x},\ref{eq:time_der_e}) converge in distribution to $x_\infty$ and $e_\infty$, respectively, where $x_\infty\sim p(x)$ and $e_\infty\sim\delta_0$ (i.e. a point mass at $0$).
\label{prop:hode}
\end{proposition}
\begin{proof}
By Theorem 1 of \citet{taghvaei19a} and Appendix C.4 of \citet{wang2019accelerated} (for an extension to high dimensional cases), 
since $\alpha_t$, $\beta_t$ and $\gamma_t$ satisfy the scaling condition in \citet{taghvaei19a} and $\log p$ is convex,
$x_t$ converges in KL divergence to $x_\infty$ and $e_t$ converges to $0$ almost surely (which implies convergence in distribution). 
Pinsker's inequality implies $x_t\rightarrow x_\infty$ in total variation, $d_\mathrm{TV}$, which has a dual representation:
$$d_\mathrm{TV}(x_t,x_\infty)=\sup_{f:\gX\rightarrow[-1,1]} \E[f(x_t)] - \E[f(x_\infty)]$$
This implies for any bounded, continuous $f$,
$$|\E[f(x_t)] - \E[f(x_\infty)]| \leq d_\mathrm{TV}(x_t,x_\infty)\cdot ||f||_\infty$$
which converges to $0$ as $t\rightarrow\infty$. 
By Portmanteau's Lemma, $x_t\rightarrow x_\infty$ in distribution.
\end{proof}

We first construct a sequence of encoding functions $m^{\text{enc}}_{n}$ and decoding functions $m^{\text{dec}}_{n}$ parameterized by neural networks, and define the following (volume preserving) invertible mappings
\begin{align*}
e^\pi_{1} &= e^\pi_0 + m^{\text{enc}}_{1}(x^\pi_0)\\
x^\pi_{n+1} &= x^\pi_{n} + 2\e \cdot m^{\text{dec}}_{{n+1}}(e^\pi_{n+1})  && \forall\,n\geq0\\
e^\pi_{n+1} &= e^\pi_n + 2\e \cdot m^{\text{enc}}_{{n+1}}(x^\pi_n) && \forall\,n\geq1
\end{align*}
with $e^\pi_0=0$ and $x^\pi_0 \sim q_0$. 
The step size parameter $\e$ will be chosen to depend on the depth coefficient $N$, i.e. the number of layers of the joint transformation. 

Below we prove ANF of the above form can universally transform $q(x)\delta_0(e)$ into $p(x)\delta_0(e)$. 
We make the following assumption on the family of $q$:
\begin{assumption}
\label{assumption}
We assume the gradient of the convex function in Proposition (\ref{prop:hode}) $\nabla\Phi$ is continuous, and that
$f(e,t):=e^{\alpha_t-\gamma_t}e$ and 
$g(x,t):=-e^{\alpha_t+\beta_t+\gamma_t}\log\frac{q_t(x)}{p(x)}$ have a bounded second time derivative (on the trajectories $x_t$ and $e_t$ which are also functions of time), and are uniformly Lipschitz; that is,
$$\max\left\{\,
||f''||, \quad
||g''||, \quad
\sup_{e\neq e', t>0}\frac{||f(e, t) - f(e', t)||}{||e-e'||}, 
\sup_{x\neq x', t>0}\frac{||g(x, t) - g(x', t)||}{||x-x'||}
\,\right\}\, \leq K$$
for some $K\geq 0$, where we define the single-argument vector functions $f(t)=f(e_t,t)$ and $g(t)=g(x_t,t)$ as the time derivatives of the trajectories $(x_t,e_t)$. 
\end{assumption}

We denote by $\gQ$ the family of probability measures that satisfies this assumption. 

Before we move on to approximation, we start with a lemma for bounding approximation error by solving recursion using the technique of generating functions. 
\begin{lemma}
\label{lem:recursive_error}
If for any $N>0$, $\{d_n: 0\leq n\leq N\}$ is a sequence of real numbers satisfying
\begin{align}
d_n \leq 
\frac{c}{N^2} +
\frac{c}{N^2} \sum_{t=1}^{n-1}\sum_{s=1}^{t} d_s \nonumber
\end{align}
for some constant $c$, then $$\max_{0\leq n\leq N}d_n \,\,\rightarrow\,\, 0\quad \text{ as } \quad N\rightarrow\infty$$
\end{lemma}
\begin{proof}
We would like to bound the error $d_n$ explicitly.
To do so, we first note that the sequence $\{d_n\}$ is no larger than $\{D_n\}$, which is recursively defined as 
\begin{align}
D_0 &= 0 \nonumber\\
D_{n+1} &=
C +
C \sum_{t=1}^{n}\sum_{s=1}^{t} D_s
\label{rec_err:xx}
\end{align}
for $n\geq0$,
where for simplicity we let $C=c/N^2$.

Now to express $D_{n+1}$ explicitly, we use the method of generating function, following the recipe of \citet{wilf2005generatingfunctionology} (see Chapter 1 for a brief introduction). 
Define function $f$ to be a power series whose coefficients are $D_n$'s; that is, $f(x)=\sum_{n\geq0}D_n x^n$. 
Multiply both sides of (\ref{rec_err:xx}) by $x^n$ and summing over the indices of non-negative integers $n\geq0$ give us
$$\frac{f(x)}{x} = \frac{C}{1-x} + \frac{Cf(x)}{(1-x)^2}$$

After rearrangement, we have
\begin{align}
f(x)\left(\frac{x^2-(2+C)x+1 }{x(1-x)^2}\right) = \frac{C}{1-x}
\quad\Rightarrow\quad
\frac{f(x)}{x} = \frac{C(1-x)}{x^2-(2+C)+1}
\nonumber
\end{align}
which can be decomposed into the partial fractions
\begin{align}
\frac{f(x)}{x} = \frac{\frac{C}{1+a_2}}{a_1-x} + \frac{\frac{C}{1+a_1}}{a_2-x}
\label{eq:partial_fraction}
\end{align}
where $a_1$ and $a_2$ are the roots of the quadratic function $x^2-(2+C)x+1$,
which satisfy $a_1+a_2=2+C$ and $a_1a_2=1$. 

For sufficiently small $x$, we can break (\ref{eq:partial_fraction}) into the geometric series
$$\frac{f(x)}{x} =
\frac{C}{a_1(1+a_2)}
\sum_{n\geq0} \left(\frac{x}{a_1}\right)^n
+
\frac{C}{a_2(1+a_1)}
\sum_{n\geq0}
\left(\frac{x}{a_2}\right)^n
$$

This means for $n>0$, since $a_1a_2=1$, the coefficient of $f(x)$ can be expressed as 
\begin{align}
D_n = 
\frac{C}{1+a_2}\frac{1}{a_1^{n}} + \frac{C}{1+a_1}\frac{(a_1a_2)^{n}}{a_2^{n}} =
C \left(
\frac{1}{(1+a_1)a_1^{n-1}} + \frac{a_1^{n}}{1+a_1}
\right)
\label{eq:Ex}
\end{align}

Now let $a_1$ be the larger root. 
Solving $x^2-(2+C)x+1$ yields 
$$a_1
= \frac{2+C + \sqrt{C^2 + 4 C}}{2}
=: 1+r$$
where $r:=\frac{C}{2}+\sqrt{\frac{C^2}{4} + C}$.

We show that the parenthesis in (\ref{eq:Ex}) can be controlled asymptotically (i.e. does not exceed certain constant for sufficiently large $N$), and that since $C$ diminishes, $D_n$ converges. 
First, since $r>0$, $a_1>1$ and
$$\frac{1}{(1+a_1)a_1^{n-1}} < \frac{1}{2}$$

Second, since $(1+r)^n\leq e^{nr}$ for $n\geq0$ and $r\geq-1$,
\begin{align*}
a_1^n=(1+r)^{n}
&\leq e^{nr} \\
&\leq \exp\left({\frac{CN}{2} + \sqrt{\frac{C^2N^2}{4} + CN^2}}\right) \\
&= \exp\left({\frac{c}{2N} + \sqrt{\frac{c^2}{4N^2} + c}}\right)
\end{align*}
which converges to $\exp(\sqrt{c})$ as $N\rightarrow\infty$. 

Finally, since $C\rightarrow 0$ as $N\rightarrow\infty$ and $d_n\leq D_n$, $d_n\rightarrow0$ for all $n\leq N$ as $N\rightarrow\infty$. 
\end{proof}

We are now ready to show the result of the pointwise approximation of the Hamiltonian ODE using ANFs with affine (more specifically, additive) coupling. 
\begin{proposition}
\label{prop:anf_pointwise}
Let $x_t$ and $e_t$ be trajectories (mappings of $x_0\in\gX=\R^d$) following the Hamiltonian ODE (\ref{eq:time_der_x},\ref{eq:time_der_e}) described in Proposition \ref{prop:hode} dependent on some initial distribution $q_0 \in \gQ$. 
For each $T>0$, we can choose some number of layers $N$ of the joint transformation and a sequence of pairs of $m^\text{enc}_n$ and $m^\text{dec}_n$ (dependent on $T$) for $1\leq n\leq N$, such that $||x^\pi_N - x_T||\rightarrow0$ and $||e^\pi_N - e_T||\rightarrow0$ as $T\rightarrow\infty$ pointwise for $x_0\in\gX=\R^d$. 
\end{proposition}
\begin{proof}
Fix $q_0\in\gQ$ and $T>0$ and some compact subset $\gX_0\subset\gX$.
We first consider all points $x_0$ in $\gX_0$, and show that $(x^\pi_n, e^\pi_n)$ can be used to approximate $(x_T,e_T)$ uniformly well. 

We consider a $N$-step joint transformation, and set $\e=\frac{T}{2N}>0$. 
We start with approximating $e_{\e}$ by $e^\pi_1$.
Since $e^\pi_0$ is $0$, by the universal approximation theorem (UAT) of neural networks \citep{cybenko1989approximation}, we can choose some $m^\text{enc}_1$ such that $||e_\e - e^\pi_1||=||e_\e - m^\text{enc}_1||\leq \e^2$ for all $x_0\in\gX_0$.

We proceed with an approximate leap-frog integration of the dynamic, using the neural encoders and decoders to approximate the time derivatives. 
Let $\gE_1:=e^\pi_1(\gX_0)$ where $e^\pi_1:=m^\text{enc}_1$, which is compact, since $\gX_0$ is compact and $e^\pi_1$ is continuous wrt $\gX_0$.
Again, by the UAT, we can choose some $m^\text{dec}_1$ such that $||f(e,\e)-m^\text{dec}_1(e)||<\e^2$ for all $e\in\gE_1$. 
Likewise, we let $\gX_1:=x^\pi_1(\gX_0)$ where $x^\pi_1:=(2\e m^\text{dec}_1 \circ e^\pi_1  + Id)(\gX_0)$ with $Id$ being the identity map, such that $\gX_1$ is also compact since $x^\pi_1$ is continuous wrt $\gX_0$, and choose $m^\text{enc}_2$ such that $||g(x,2\e)-m^\text{enc}_2(x)||<\e^2$ for all $x\in\gX_1$.

Repeating the same construction for $m^\text{dec}_n$ and $m^\text{enc}_n$ for $n\leq N$, 
we have
\begin{align}
x^\pi_{n+1} &= x^\pi_{n} + 2\e m^{\text{dec}}_{{n+1}}(e^\pi_{n+1}) \label{eq:x_nn}\\
e^\pi_{n+1} &= e^\pi_n + 2\e m^{\text{enc}}_{{n+1}}(x^\pi_n) \label{eq:e_nn}
\end{align}
with $m^\text{dec}_n$ and $m^\text{enc}_n$ chosen such that 
\begin{enumerate}
    \item $||f(e,2n\e+\e)-m^\text{dec}_{n+1}(e)||<\e^2$ for all $e\in\gE_{n+1}:=e^\pi_{n+1}(\gX_0)$ where $e^\pi_{n+1}:=2\e m^\text{enc}_{n+1}\circ x^\pi_n + e^\pi_n$ is a continuous map of $\gX_0$; and
    \item $||g(x,2n\e)-m^\text{enc}_{n+1}(x)||<\e^2$ for all $x\in\gX_n:=x^\pi_n(\gX_0)$ where $x^\pi_n:=2\e m^\text{dec}_n\circ e^\pi_n + x^\pi_{n-1}$ is a continuous map of $\gX_0$.
\end{enumerate}
Such choices of $m^\text{enc}_n$ and $m^\text{dec}_n$ are possible since by construction $\gX_{n-1}$ and $\gE_n$ are compact. 

Equations (\ref{eq:x_nn},\ref{eq:e_nn}) are approximate midpoint methods as they use functions to approximate the time derivatives evaluated at midpoints of their counterparts. 
The exact midpoint method has a cubic error rate of $\frac{h^3}{24} f''(\xi)$, for some $\xi$ between the midpoint and the approximating point, where $h$ is the interval width of each iteration; see Section 5.4 of \citet{epperson2013introduction}. 
That is,
\begin{align}
x_{2n\e + 2\e} &= x_{2n\e} + 2\e f(e_{2n\e+\e}, 2n\e+\e) + \frac{\e^3}{3}f''(\xi^x_{n+1}) \label{eq:x_mid}
\intertext{for some $\xi^x_{n+1}$ between the two steps.
Similarly,}
e_{2n\e + \e} &= e_{2n\e-\e} + 2\e g(x_{2n\e}, 2n\e) + \frac{\e^3}{3}g''(\xi^e_{n+1}) \label{eq:e_mid}
\end{align}
for some $\xi^e_{n+1}$ between the two steps. 

Subtracting (\ref{eq:x_nn}) from (\ref{eq:x_mid}) yields
$$x_{2n\e+2\e} - x^\pi_{n+1} = x_{2n\e} - x^\pi_n + 2\e f(e_{2n\e+\e}, 2n\e+\e) - 2\e m^{\text{dec}}_{{n+1}}(e^\pi_{n+1}) + \frac{\e^3}{3}f''(\xi^x_{n+1})$$

By triangle inequality, we have
\begin{align*}
\norm{x_{2n\e+2\e} - x^\pi_{n+1} } 
&\leq \norm{ x_{2n\e} - x^\pi_n } + 
\norm{ 2\e f(e_{2n\e+\e}, 2n\e+\e) - 2\e m^{\text{dec}}_{{n+1}}(e^\pi_{n+1}) } + 
\norm{ \frac{\e^3}{3}f''(\xi^x_{n+1}) } \\
&\leq 
\underbrace{\norm{ x_{2n\e} - x^\pi_n } + 
2\e\norm{ f(e_{2n\e+\e}, 2n\e+\e) - m^{\text{dec}}_{{n+1}}(e^\pi_{n+1}) }}_{\text{propagated error}} + 
\underbrace{\frac{\e^3}{3}\norm{ f''(\xi^x_{n+1}) }}_{\text{truncated error}}
\end{align*}

The error on the RHS consists of two parts: (1) the first two terms constitute the propagated error from the previous steps and (2) the third term is a newly introduced truncation error due to the Taylor expansion. 

By triangle inequality again,

\resizebox{1.00\linewidth}{!}{
\begin{minipage}{\linewidth}
\begin{align*}
\norm{ f(e_{2n\e+\e}, 2n\e+\e) - m^{\text{dec}}_{{n+1}}(e^\pi_{n+1}) } 
&= \norm{ f(e_{2n\e+\e}, 2n\e+\e) 
- f(e^\pi_{n+1}, 2n\e+\e) 
+ f(e^\pi_{n+1}, 2n\e+\e) 
- m^{\text{dec}}_{{n+1}}(e^\pi_{n+1}) } \\
&\leq 
\underbrace{\norm{ f(e_{2n\e+\e}, 2n\e+\e) 
- f(e^\pi_{n+1}, 2n\e+\e) }}_{\text{midpoint deviation}}
+ \underbrace{\norm{ f(e^\pi_{n+1}, 2n\e+\e) 
- m^{\text{dec}}_{{n+1}}(e^\pi_{n+1}) }}_{\text{approximation error}}
\end{align*}
\end{minipage}
}

Again the RHS can be decomposed into two error parts: (1) a midpoint deviation resulting from performing midpoint numerical integration which would not vanish even if the neural network is replaced with the true time derivative, and (2) an approximation error due to the inaccuracy of approximating the time derivative. 

Letting $d^x_n=||x_{2n\e}-x^\pi_n||$ and $d^e_n=||e_{2n\e-\e}-e^\pi_n||$, and applying the properties of the Assumption \ref{assumption}, we have
$$d^x_{n+1}\leq d^x_n + 2\e(Kd^e_{n+1} + \e^2) + \frac{\e^3K}{3}
=d^x_n + 2\e Kd^e_{n+1} + \e^3\left(\frac{K}{3}+2\right)$$
owing to the uniform error bound of the neural decoder
$||f(e,2n\e+\e)-m^\text{dec}_{n+1}(e)||<\e^2$ for all $e\in\gE_{n+1}$ and the fact that $e^\pi_{n+1}(x_0)\in\gE_{n+1}$ since $x_0\in\gX_0$.  

The same can be done to obtain a bound on $d^e_{n+1}$ by subtracting (\ref{eq:e_nn}) from (\ref{eq:e_mid}), which yields
$$d^e_{n+1}\leq d^e_n + 2\e Kd^x_n + \e^3\left(\frac{K}{3}+2\right)$$

To summarize, we have
\begin{align}
d^e_1 &\leq \e^2 \\
d^x_{n+1} &\leq d^x_n + 2\e K'd^e_{n+1} + \e^3K' &&\text{ for }n\geq0\\
d^e_{n+1} &\leq d^e_n + 2\e K'd^x_n + \e^3K' &&\text{ for }n\geq1
\end{align}
where $K'=\max\{K, \frac{K}{3}+2\}$.

Summing $d^x_{1},...,d^x_{n}$ and subtracting $d^x_{1}+...+d^x_{n-1}$ from both sides yield
\begin{align}
d^x_n \leq 2\e K' \sum_{t=1}^n d^e_t + n\e^3K' \label{err:xe}
\end{align}
Note that $d^x_0=0$. 
Similarly, summing $d^e_{2},...,d^e_{n}$ and subtracting $d^e_{2}+...+d^e_{n-1}$ from both sides yield
\begin{align}
d^e_n \leq d^e_1 + 2\e K' \sum_{t=1}^{n-1} d^x_t + (n-1)\e^3K' \label{err:ex}
\end{align}

To recursively express $d^x_n$ in terms of itself (except for $d^e_1$), 
we sum over the sequence $d^e_{1},...,d^e_{n}$ again 
$$\sum_{t=1}^n d^e_t \leq n d^e_1 + 2\e K' \sum_{t=2}^n\sum_{s=1}^{t-1} d^x_s + \sum_{t=1}^n (t-1)\e^3K'$$

Substituting into (\ref{err:xe}) yields
\begin{align}
d^x_n \leq 2\e K' \left(n d^e_1 + 2\e K' \sum_{t=2}^n\sum_{s=1}^{t-1} d^x_s + \sum_{t=1}^n (t-1)\e^3K'\right) + n\e^3K' \nonumber
\end{align}

Since $n\leq N$,
$\sum_{t=1}^n t \leq n^2$,
$d^e_1\leq\e^2$ and $\e=\frac{T}{2N}$, the above can be rearranged and further bounded by
\begin{align}
d^x_n \leq 
\left(
\frac{T^3K'^2}{4} + 
\frac{T^4K'^2}{8} + 
\frac{T^3K'}{8}
\right)\frac{1}{N^2} +
\frac{T^2K'^2}{N^2} \sum_{t=1}^{n-1}\sum_{s=1}^{t} d^x_s 
\end{align}

The same can be done for (\ref{err:ex}) to analyze $d^e_n$. 

$$\sum_{t=1}^{n-1}d^x_t \leq 2\e K'\sum_{t=1}^{n-1}\sum_{s=1}^td^e_s + \sum_{t=1}^{n-1} t\e^3K'$$
\begin{align}
d^e_n \leq d^e_1 + 2\e K' 
\left(
2\e K'\sum_{t=1}^{n-1}\sum_{s=1}^td^e_s + \sum_{t=1}^{n-1} t\e^3K'
\right)
+ (n-1)\e^3K' \nonumber
\end{align}
\begin{align}
d^e_n \leq 
\left(
\frac{T^2}{4} + 
\frac{T^4K'^2}{8} + 
\frac{T^3K'}{8}
\right)\frac{1}{N^2} +
\frac{T^2K'^2}{N^2} \sum_{t=1}^{n-1}\sum_{s=1}^{t} d^e_s 
\end{align}

By Lemma \ref{lem:recursive_error}, we know that the elements of both sequences of error $d^x_n$ and $d^e_n$ converge uniformly on $1\leq n\leq N$ to $0$ as $N\rightarrow\infty$. 
In particular, for all $T>0$, $\delta>0$ and compact subset $\gX_0$ of $\R^d$, there exists some large enough integer $N(T,\delta,\gX_0)>0$ for which a joint transformation of $N(T,\delta,\gX_0)$ layers parameterized by some neural encoders and decoders satisfies
$d^x_{N(T,\delta,\gX_0)} \leq \delta$ and $d^e_{N(T,\delta,\gX_0)} \leq \delta$ for all $x_0\in\gX_0$.

Consider some positive value $B>0$. 
We let $\gX_0=[-B,B]^d$, $T=B$ and $\delta=\frac{1}{B}$. 
We can find a sequence of models with an error rate $d^x_{N(B,1/B,[-B,B]^d)} \leq 1/B$ and $d^e_{N(B,1/B,[-B,B]^d)} \leq 1/B$ converging pointwise on $\R^d$ to $0$ as $B\rightarrow\infty$.
This implies 
$$d^x_{N(B,1/B,[-B,B]^d)} = \norm{x_{B}-x^\pi_{N(B,1/B,[-B,B]^d)}}\rightarrow0$$
pointwise as $B\rightarrow\infty$. 
The same holds for the augmented variable $e$. 
\end{proof}

The lemma below shows if one can approximate the solution of an ODE ($||y_n-x_n||\rightarrow0$, i.e. $x_n$ and $y_n$ are asymptotically indistinguishable) and if the limit of the solution is a transport map ($x_n\overset{d}{\rightarrow} x_\infty$), then the approximation also forms a transport map ($y_n\overset{d}{\rightarrow} x_\infty$). 
\begin{lemma}
\label{lem:conv_indistinguishable}
Let $x_\infty$, $(x_n:n\geq0)$ and $(y_n:n\geq0)$ be random variables. 
If $x_n\rightarrow x_\infty$ in distribution and if $||x_n - y_n||\rightarrow0$ almost surely as $n\rightarrow\infty$, then $y_n\rightarrow x_\infty$ in distribution. 
\end{lemma}
\begin{proof}
Let $\Lambda:\R^d\rightarrow\R$ be an arbitrary \emph{bounded} and \emph{Lipschitz continuous} function. 
Then 
\begin{align*}
\left|\E\left[\Lambda\left(x_{\infty}\right)-\Lambda\left(y_n\right)\right]\right|
&\leq 
\left|\E\left[\Lambda\left(x_{\infty}\right) -
\Lambda(x_n) +
\Lambda(x_n) -
\Lambda\left(y_n\right)\right]\right| \\
&\leq
\left|\E\left[\Lambda\left(x_{\infty}\right) -
\Lambda(x_n)\right]\right|
+
\E\left[\left|\Lambda\left(x_{n}\right)-\Lambda\left(y_n\right)\right|\right] 
\end{align*}

First, since $x_n\rightarrow x_\infty$ in distribution and since $\Lambda$ is bounded and continuous, by the Portmanteau Lemma the first term of the RHS converges to $0$ as $n\rightarrow\infty$. 
Second,
since $y_n$ is almost surely asymptotically indistinguishable from $x_n$ (let $\Omega$ be the almost sure set),
and since the Lipschitzness of $\Lambda$ implies uniform continuity, the following are true
\begin{itemize}
    \item For all $\epsilon>0$, there exists a $\delta>0$ such that $||x-y||\leq \delta$ implies $|\Lambda(x)-\Lambda(y)|\leq \epsilon$.
    \item For any $\delta>0$, there exists a integer $N>0$ such that for all $n\geq N$, $||x_n-y_n||\leq \delta$ for all $\omega\in\Omega$. 
\end{itemize}
These imply $||\Lambda(x_n)-\Lambda(y_n)||\rightarrow0$ on $\Omega$.
Then
$$\E\left[\left|\Lambda\left(x_{n}\right)-\Lambda\left(y_n\right)\right|\right]=
\underbrace{\E_{\Omega}\left[\left|\Lambda\left(x_{n}\right)-\Lambda\left(y_n\right)\right|\right]}_{E_1}+
\underbrace{\E_{\Omega^c}\left[\left|\Lambda\left(x_{n}\right)-\Lambda\left(y_n\right)\right|\right]}_{E_2}$$
converges to $0$, since (1)
boundedness of $\Lambda$ and the \emph{Bounded Convergence Theorem} imply $E_1\rightarrow0$ and (2) $\sup_{x}\Lambda(x)<\infty$ implies $E_2\leq 2\sup_{x}\Lambda(x) \sP(\Omega^c)$ = 0. 
Finally, since $\Lambda$ is arbitrary,
by the Portmanteau Lemma again, $y_n$ converges in distribution to $x_\infty$ as $n\rightarrow\infty$. 
\end{proof}

We now are ready to prove Theorem \ref{thm:anf_dist}, which we restate below.
The main idea is to notice that ANFs can be made pointwise inseparable from the Hamiltonian ODE, which implies weak convergence since the Hamiltonian ODE converges in distribution. 
\anfdist*
\begin{proof}
First, by Proposition \ref{prop:hode}, $x_B\rightarrow x_\infty$ in distribution as $B\rightarrow\infty$. 
Second, $x_{B}$ and $x^\pi_{N(B,1/B,[-B,B]^d)}$ chosen from Proposition \ref{prop:anf_pointwise} are almost surely asymptotically indistinguishable.
Thus, by Lemma \ref{lem:conv_indistinguishable}, $x^\pi_{N(B,1/B,[-B,B]^d)}$ converges in distribution to $x_\infty$.
The same holds for the augmented variable $e$. 
Let $(x^\pi_N)$ and $(e^\pi_N)$ denote such sequences. 
By Theorem 2.7 of \citet{van2000asymptotic}, $(x^\pi_N,e^\pi_N)\rightarrow(x_\infty,e_\infty)$ in distribution (as $e_\infty=0$ is a constant).
\end{proof}


\end{document}